\documentclass{article}

\PassOptionsToPackage{numbers}{natbib}

\usepackage[final]{neurips_2022}

\usepackage[utf8]{inputenc} 
\usepackage[T1]{fontenc}    
\usepackage{hyperref}       
\usepackage{url}            
\usepackage{algorithm,algorithmic}
\usepackage{amsthm}
\usepackage{amsmath}
\usepackage{amssymb}
\usepackage{graphicx}
\usepackage{booktabs}       
\usepackage{amsfonts}       
\usepackage{nicefrac}       
\usepackage{microtype}      
\usepackage[table]{xcolor}

\usepackage{subcaption}
\usepackage{wrapfig}
\usepackage[english]{babel}
\usepackage{stmaryrd}

\usepackage{multirow}
\usepackage{float}
\usepackage{longtable,tabu}
\usepackage{siunitx}

\usepackage{cleveref}
\usepackage{autonum} 

\definecolor{greend}{HTML}{118E1E}

\newcommand*\numcircledmod[1]{\raisebox{.5pt}{\textcircled{\raisebox{-.9pt} {#1}}}}

\newcommand*{\defeq}{\stackrel{\text{def}}{=}}
\DeclareMathOperator{\sign}{sgn}

\newtheorem{proposition}{Proposition}[section]

\newtheorem{lemma}{Lemma}[section]

\newcommand{\secYJ}{\textsc{SecureFedYJ}}
\newcommand{\expYJ}{\textsc{ExpYJ}}
\newcommand{\fedavg}{\textsc{FedAvg}}
\definecolor{LightRed}{rgb}{0.87, 0.36, 0.51}

\title{SecureFedYJ: a safe feature Gaussianization protocol for Federated Learning}

\author{%
    Tanguy~Marchand \\
    Owkin Inc., New York, USA.\\
  \texttt{tanguy.marchand@owkin.com} 
   \And
    Boris~Muzellec\\
    Owkin Inc., New York, USA.\\
  \texttt{boris.muzellec@owkin.com} 
   \AND
    Constance~Beguier\thanks{Contribution done while at Owkin, Inc.}\\
     \And
    Jean~Ogier du Terrail\\
    Owkin Inc., New York, USA.\\
  \texttt{jean.du-terrail@owkin.com } 
     \And
    Mathieu~Andreux\\
    Owkin Inc., New York, USA.\\
  \texttt{mathieu.andreux.com} 
}

\begin{document}

\maketitle

\begin{abstract}
The Yeo-Johnson (YJ) transformation is a standard parametrized per-feature unidimensional 
transformation often used to Gaussianize features in machine learning. 
In this paper, we investigate the problem of applying the YJ transformation 
in a cross-silo Federated Learning setting under privacy constraints. For the first time, we prove that the 
YJ negative log-likelihood is in fact convex, which allows us to optimize 
it with exponential search. We numerically show that the resulting algorithm is 
more stable than the state-of-the-art approach based on the Brent minimization method. Building on 
this simple algorithm and Secure Multiparty Computation routines, we propose \secYJ, 
a federated algorithm that performs a pooled-equivalent YJ transformation 
without leaking more information than the final fitted parameters do. Quantitative  
experiments on real data demonstrate that, in addition to being secure, our approach
reliably normalizes features across silos as well as if data were pooled, making
it a viable approach for safe federated feature Gaussianization.
\end{abstract}

\section{Introduction}\label{sec:introduction}
Federated Learning (FL)~\cite{shokri2015privacy, mcmahan2017communication} is an approach that was recently proposed 
to train machine learning (ML) models across multiple data holders, or {\it clients}, without centralizing
data points, notably for privacy reasons.
While many FL applications have been proposed, two main settings
have emerged \cite{kairouz2019advances}: cross-device FL, involving a large
number of small edge devices, and cross-silo FL, dealing with a smaller number of clients, with larger
computational capabilities.
Due to the sensitivity and relative local scarcity of medical data,
healthcare is a promising application of cross-silo FL~\cite{rieke2020future},
e.g.\ to train a biomedical ML model between different hospitals
as if all the datasets were pooled in a central server.
In this paper, we focus on the cross-silo setting.

\paragraph{The constraints of cross-silo FL}
Although cross-silo FL resembles standard distributed learning,
it faces at least two important distinct challenges:
privacy and heterogeneity.
Due to data sensitivity, clients might impose stringent security and
privacy constraints on FL collaborations. This arises in
\textit{coopetitive}
FL projects, where models are jointly trained on industrial competitors' datasets~\cite{zheng2019helen},
as well as medical FL applications, where conservative data regulations might apply. In this setting, using standard FL algorithms
such as~\fedavg~\cite{mcmahan2017communication} might not provide enough privacy
guarantees, as privacy attacks such as data reconstruction can be carried out based on the
clients' gradients~\cite{zhu2020deep, zhao2020idlg}. Various protocols based on
Secure Multiparty Computation~(SMC) (see \Cref{sec:background} for more details),
such as Secure Aggregation~\cite{bonawitz2017practical},
can mitigate this shortcoming by disclosing only the sum
of the gradients from all clients to the server, without disclosing each gradient individually.

An additional constraint is that data might present statistical heterogeneity across clients,
i.e.\ the local clients' data distributions may not be identical.
In the case of medical applications, such heterogeneity may be caused e.g.\
by environmental variations or differences in the material that was used for acquisition~\cite{shafiq2017intrinsic, tellez2019quantifying, andreux2020siloed}.
While different ways of adapting federated training algorithms
have been proposed to automatically tackle heterogeneity~\cite{li2020federated, li2020federatedScaffold, karimireddy2020scaffold}, these solutions
do not address data harmonization and normalization prior to FL training.

\paragraph{Preprocessing in ML}
Data preprocessing is a crucial step in many ML applications, leading to important performance gains. 
Among others, common preprocessing methods include data whitening, principal component analysis (PCA)~\cite{jolliffe2005principal} or zero component analysis
\cite{krizhevsky2009learning, goodfellow2013maxout, springenberg2014striving}.
However, linear normalization methods might not suffice when the original data distribution is highly non-Gaussian.
For tabular and time series data,
a popular approach to Gaussianize the marginal distributions is to apply feature-wise non-linear transformations. Two commonly-used
parametric methods are the Box-Cox \cite{box1964analysis} transformation and its
extension, the Yeo-Johnson (YJ) transformation \cite{yeo2000new}.
Both have been used in multiple applications, such as climate and weather
forecast~\cite{zhang2009impact, wang2011multisite, wang2012merging},
economics~\cite{das2017economic} and genomic
studies~\cite{brunner2013immune, zwiener2014transforming, chien2020rank}.

\paragraph{Problem and contributions}
In this paper, we investigate the problem of data normalization
in the cross-silo FL setting, by exploring how to apply the YJ
transformation to a distributed dataset.
This problem arises frequently in medical cross-silo FL, e.g.\ when trying to
jointly train models on genetic data (see e.g. \cite{fro2021nature,zolotareva2021flimma}).
Due to data heterogeneity,
no single client can act as a reference client: indeed, there is no guarantee that transformation parameters fitted on a single client would be relevant for other clients' data. Hence, it is necessary to fit normalization methods on the full federated dataset.
Moreover, in this setting, data privacy is of paramount importance, and
therefore FL protocols should be carefully designed.
Our main contributions to this problem are as follows:
\begin{enumerate}
\item We prove that the negative YJ log-likelihood is convex (\Cref{sec:expYJ}),
which is a novel result, to the best of our knowledge.
\item Building on this property, we introduce \expYJ,  a method to fit the YJ transformation based
on exponential search (\Cref{sec:expYJ}). We numerically show that this method
is more stable than standard approaches for fitting the YJ transformation based
on the Brent minimization method \cite{brent1973algorithms}.
\item We propose \secYJ\ (\Cref{sec:secYJ}),
a secure way to extend \expYJ\ in the cross-silo FL setting using SMC.
We show that \secYJ\ does not leak any information
on the datasets apart from what is leaked by the parameters minimizing the YJ negative
log-likelihood (\Cref{sec:privacyleakage} and \Cref{prop:leakage}).
By construction, \secYJ\ provides the same results as the pooled-equivalent \expYJ, regardless of how the data is split across the clients. We check this property in numerical experiments
(\Cref{subsec:sanitycheck}). The core ideas behind the resulting algorithm, \secYJ, are summarised in \Cref{schemaYJ}.
\end{enumerate}

Finally, we illustrate our contributions in numerical applications on synthetic and genomic data in \Cref{sec:applications}.

\section{Background}\label{sec:background}
\paragraph{The Yeo-Johnson transformation}
The YJ transformation \cite{yeo2000new} was introduced in order to
Gaussianize data that can be either positive or negative. It was proposed as a generalization of the Box-Cox transformation \cite{box1964analysis},
that only applies to non-negative data.
The YJ transformation consists in applying to each feature a monotonic function $\Psi (\lambda, \cdot)$ parametrized by a scalar $\lambda$, independently of the other features. Thus, there are as many $\lambda$'s as there are features. For a real number $x$, $\Psi (\lambda, x)$ is defined as:
 \begin{equation}\label{eq:YJanalytics}
  \Psi(\lambda, x) =  \begin{cases}
        [(x+1)^\lambda -1 ]/ \lambda, & \text{if } x \geq 0, \lambda \neq 0,\\
        \ln(x+1), & \text{if } x \geq 0, \lambda = 0,\\
    - [(-x+1)^{2-\lambda} -1 ]/ (2-\lambda), & \text{if } x < 0, \lambda \neq 2,\\
   - \ln(-x+1), & \text{if } x < 0, \lambda = 2.\\
        \end{cases}
\end{equation}
\Cref{fig:yj_func} shows the shape of the YJ function for various values of $\lambda$.

\paragraph{The Yeo-Johnson likelihood}
Let us consider real-valued samples~$\{x_i\}_{i=1,\cdots, n}$, and let us
apply the YJ transformation~$\Psi(\lambda, \cdot)$ to these samples to Gaussianize their distribution.
The log-likelihood that~$\{\Psi(\lambda, x_i)\}_{i=1,\cdots,n}$ comes from a
Gaussian with mean~$\mu$ and variance~$\sigma^2$ is given by (derivation details are provided in \Cref{app:likelihood}):
\begin{equation}\label{eq:loglikelihoodYJ1}
\log \mathcal{L}_{\mathrm{YJ}}(\lambda, \sigma^2, \mu) =
- \frac{n}{2}\log(2\pi\sigma^2)
- \frac{1}{2\sigma^2}\sum_{i=1}^n\left[\Psi(\lambda, x_i)-\mu\right]^2 
 + (\lambda-1)\sum_{i=1}^n \sign(x_i) \log (|x_i|+1).
\end{equation}
For a given $\lambda$, the log-likelihood is maximized
for~$\mu_* =  \frac{1}{n}\sum_{i=1}^n \Psi(\lambda, x_i)$ 
and~$\sigma^2_* = \frac{1}{n}\sum_{i=1}^n (\Psi(x_i, \lambda) - \mu_*)^2$.
Once we replace~$\mu$ and~$\sigma^2$ by $\mu_*$ and~$\sigma^2_*$, it becomes:
\begin{equation}
\label{eq:loglikelihoodYJ2}
\log \mathcal{L}_{\mathrm{YJ}}(\lambda) =  - \frac{n}{2}\log( \sigma^2_{\Psi(\lambda, \{x_i\})}) + (\lambda-1)\sum_{i=1}^n \sign(x_i) \log (|x_i|+1) - \frac{n}{2}\log (2\pi),
\end{equation}
see ~\cite{yeo2000new}. Maximizing the YJ log-likelihood is therefore a 1-dimensional problem for each feature.
Once the optimal $\lambda_*$ is found,
the transformed data $\Psi (\lambda, x_i)$ is usually renormalized by
subtracting its empirical mean $\mu_*$
and dividing by the square root of its empirical variance
$\sigma^2_*$.
Figure \ref{yj_example} shows an example of the YJ transformation applied to a skew-normal distribution.
\begin{figure}
   \centering
   \begin{subfigure}[b]{0.32\textwidth}
      \centering
      \includegraphics[width=\textwidth]{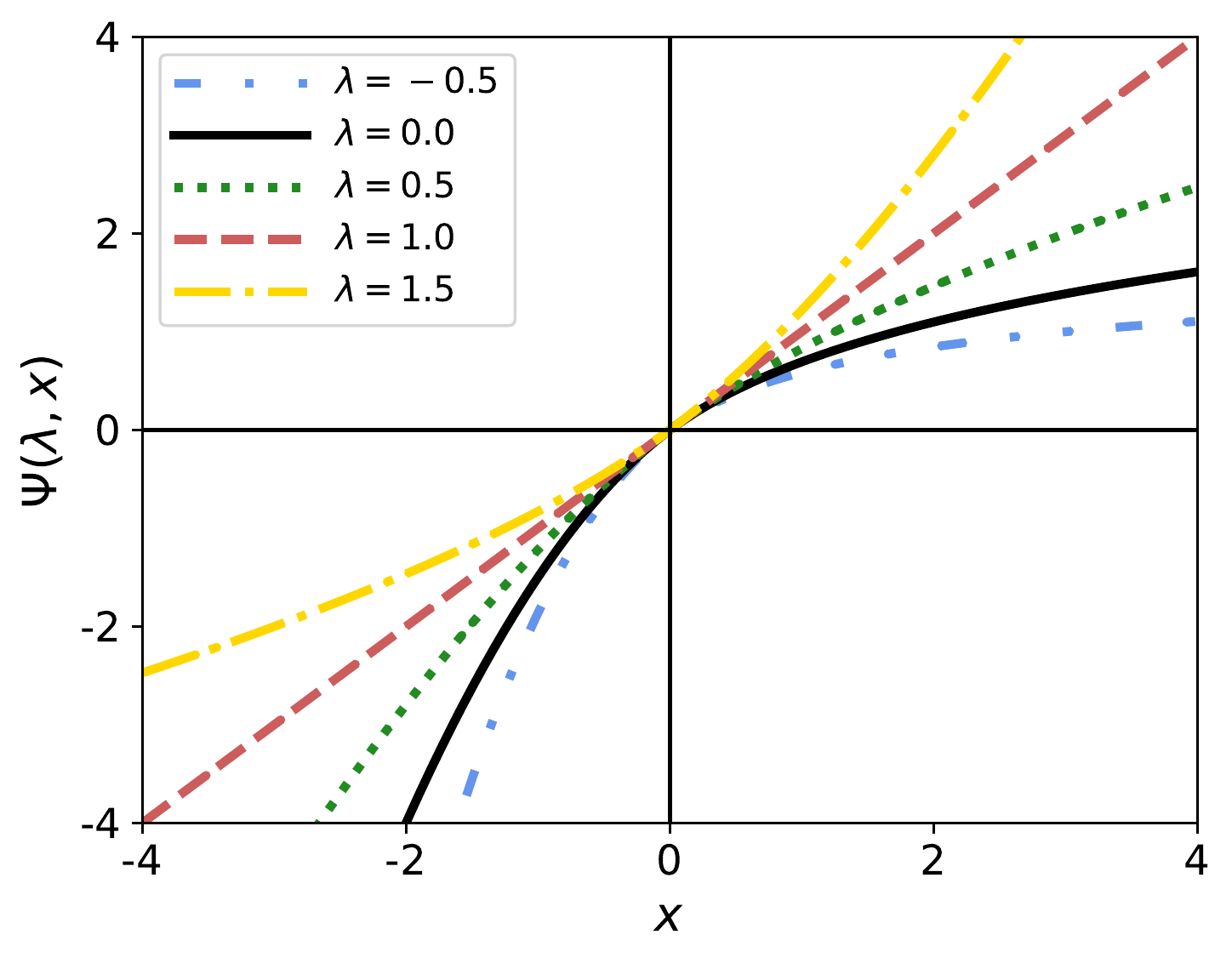}
      \caption{$\Psi(\lambda, \cdot)$ for various $\lambda$}
      \label{fig:yj_func}
   \end{subfigure}
   \begin{subfigure}[b]{0.62\textwidth}
      \includegraphics[width=\textwidth]{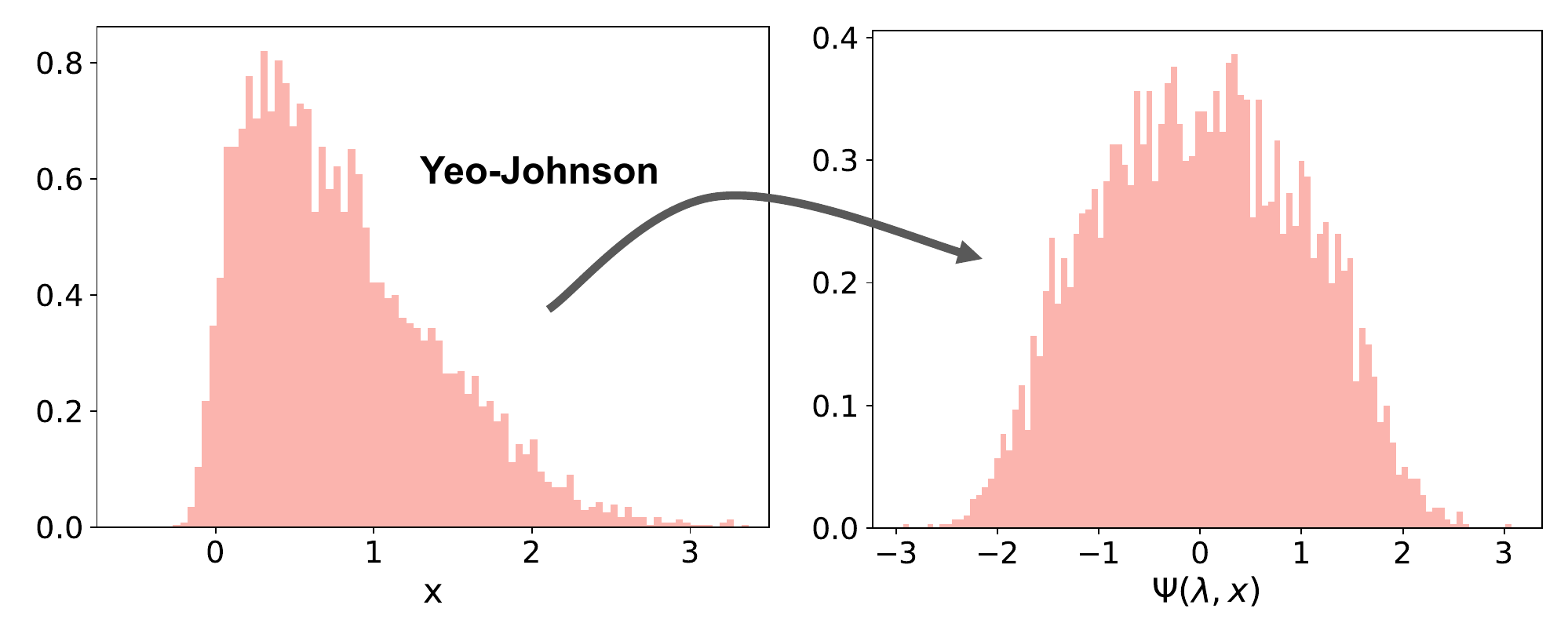}
      \caption{The YJ transformation applied to a skew-normal distribution}
      \label{yj_example}
   \end{subfigure}
   \caption{The Yeo-Johnson transformation applies a 1-D univariate transform to Gaussianize data.}
\end{figure}
Note that in a typical application, the triplet~$(\lambda_*, \mu_*, \sigma^2_*)$ is fitted
on the training data only, and is then used to Gaussianize the test dataset during inference.
\paragraph{Minimization methods in dimension 1}
As seen above, fitting a YJ transformation can be reduced to a 1D optimization problem.
To tackle this problem, we introduce two standard 1D minimization methods:
(i) Brent minimization \cite{brent1973algorithms} and (ii) exponential search \cite{bentley1976almost}.

Brent minimization \cite{brent1973algorithms} (not to be confused with the Brent-Dekker method, see \cite{brent1973algorithms}, chapters 3 and 4) is a widely used method for 1D optimization. It is
based on golden section search and successive parabolic interpolations, and does not require evaluating any derivatives.
This algorithm is guaranteed to converge to a local minimum with superlinear convergence
of order at least~1.3247.
Standard implementations of the YJ transformation, in particular the
{\it scikit-learn} implementation \cite{scikit-learn},
are based on the Brent minimization method to minimize the negative log-likelihood
provided by Eq.~\eqref{eq:loglikelihoodYJ2}.

Exponential search \cite{bentley1976almost} is a dichotomic algorithm designed for unbounded search spaces.
The idea is to first find bounds, and then to perform a classic binary search within these bounds.
This algorithm can be used to find the minimum of convex differentiable functions with linear convergence,
as explained in \Cref{app:expsearch}.
In this work, we build on exponential search to propose a federated version of the YJ transform,
for two main reasons: (i) it is more numerically stable than Brent minimization,
as shown in \Cref{sec:expYJ} and \Cref{fig:brent_exp_comparison},
(ii), it may conveniently be adapted to a federated setting, as shown in \Cref{sec:secYJ}, and
(iii), this latter federated adaptation offers strong privacy garantees, as shown by Proposition~\ref{prop:leakage}.

\paragraph{Secure Multiparty Computation}
As illustrated by various privacy gradient attacks \cite{zhu2020deep, zhao2020idlg},
sensitive information on the clients' datasets can be leaked to the central server
during an FL training. One way to mitigate this risk is to use Secure Multiparty
Computation (SMC) protocols to hide individual contributions to the server.
SMC enables one to evaluate functions with inputs distributed across different users without
revealing intermediate results and is often based on secret sharing.
SMC protocols tailored for ML use-cases have been
recently proposed~\cite{damgaard2013practical, demmler2015aby, mohassel2017secureml, riazi2018chameleon, wagh2018securenn, mohassel2018aby3, wagh2020falcon, ryffel2020ariann}.
These protocols are either designed to enhance
the privacy of FL trainings, or to perform secure inference, i.e.\ to enable the evaluation
of model trained privately on a server without revealing the data nor the model.

A popular FL algorithm relying on SMC is Secure Aggregation (SA)~\cite{bonawitz2017practical}. Schematically, in SA each client adds a random mask to their model update before sending it to the central server. These masks have been tailored in such a way that they all together sum to zero. Therefore, the central server cannot see the individual updates of the clients, but it can recover the sum of these updates by adding all the masked quantities sent from them.

More generally, an SMC routine schematically works as follows (we refer to \Cref{app:SMC} for further details).
Let us consider the setting where $K$ parties $k=1,\dots,K$ want to compute~$g = f(h^{(1)}, \dots, h^{(K)})$
for a known function $f$, where $(h^{(1)}, \dots, h^{(K)})$ denote private inputs.
Each party~$k$ knows $h^{(k)}$ and is not willing to share it.
During the first step, {\it secret sharing}, each party splits its private
input $h^{(k)}$ into K secret shares $h^{(k)}_1, \dots, h^{(k)}_K$, and sends the
shares $h^{(k)}_{k'}$ to the party~$k'$. These secret shares are constructed in
such a way that (i) knowing $h^{(k)}_{k'}$ does not provide any information on the
value of $h^{(k)}$, and (ii) $h^{(k)}$ can be reconstructed from the
vector $(h^{(k)}_1, \dots, h^{(k)}_K)$. For simplicity, we
denote $\llbracket h^{(k)} \rrbracket = (h^{(k)}_1, \dots, h^{(k)}_K)$ the vector of
share secrets. In a second step, {\it the computation}, each party $k'$ computes the quantity denoted
$g_{k'}$ using the secret shares they know along with 
intermediate quantities exchanged with the other parties.
The way to compute $g_{k'}$ depends on $f$ and on the SMC protocol that is used,
and is chosen so that $g = f(h^{(1)}, \dots, h^{(K)})$ can be reconstructed
from $(g_{1}, \dots g_{K})$. Said otherwise, $g_{k'}$ are secret shares
of $g$: $\llbracket g \rrbracket = (g_{1}, \dots g_{K})$.
Finally, during the {\it reveal} step, each party $k$ reveals $g_{k}$ to all
other parties, and each party can reconstruct $g$ from $(g_{1}, \dots g_{K})$.

\paragraph{Threat model}
In this work, we consider an honest-but-curious setting~\cite{paverd2014modelling}.
Neither the clients nor
the server will deviate from the agreed protocol, but each
party can potentially try to infer as much information as possible
using data they see during the protocol.
This setting is relevant for cross-silo FL, where participants
are often large institutions whose reputation could be ternished by
a more malicious behaviour.
\section{A novel method to optimize the Yeo-Johnson log-likelihood: \expYJ\ }
\label{sec:expYJ}

\begin{wrapfigure}{R}{0.565\textwidth}
	\vskip-.75cm
\begin{minipage}{.565\textwidth}
   \begin{algorithm}[H]
      \caption{\expYJ}
      \label{alg:dicYJ}
   \begin{algorithmic}
      \REQUIRE data $x_{i}$, total data size $n$, number of steps $t_\mathrm{max}$
      \STATE Initialize $\lambda_{t=0} \leftarrow 0$, $\lambda^+_{t=0} \leftarrow \infty$, $\lambda^-_{t=0} \leftarrow -\infty$
      \STATE Compute $ S_\varphi $
      \FOR{$t=1$ {\bfseries to} $t_\mathrm{max}$}
      \FOR{$g \in \{\Psi(\lambda, \cdot),\Psi(\lambda, \cdot)^2, \partial_\lambda \Psi(\lambda, \cdot),  \partial_\lambda\Psi(\lambda, \cdot)^2\}$}
      \STATE Compute $S_g$
         \ENDFOR
      \STATE $\Delta_t =  \sign\left[n  S_{\partial \Psi^2} - 2 S_{\Psi}  S_{\partial \Psi}  - 2 S_\varphi \left( S_{\Psi^2} -\frac{S_{\Psi} ^2}{n} \right)\right]$
      \STATE $\lambda_t, \lambda^-_t,\lambda^+_t \leftarrow \mathrm{\textsc{ExpUpdate}}(\lambda_{t-1}, \lambda^-_{t-1},\lambda^+_{t-1}, \Delta_t)$
      \ENDFOR
      \STATE $\lambda_* \leftarrow \lambda_{t_\mathrm{max}}$
      \STATE Compute $\mu_* =  S_{\Psi} / n$ and $\sigma^2_* =  S_{\Psi^2} / n - \mu_*^2$
       \ENSURE The fitted triplet ($\lambda_*, \mu_*, \sigma_*^2)$
   \end{algorithmic}
   \end{algorithm}
\end{minipage}
\vskip-.5cm
\end{wrapfigure}
In this section, we leverage the convexity of the negative log-likelihood
of the YJ transformation (see \Cref{th:convex}) to propose a new method to find the optimal $\lambda_*$ using exponential
search. While this method only offers linear convergence, compared to the super-linear convergence of Brent minimization method, we demonstrate two of its advantages: (i) it is more numerically stable,
and (ii) it is easily amenable to an FL setting with strong privacy guarantees.
The method proposed in this section is based on the following result.
\begin{proposition}
\label{th:convex}
The negative log-likehood $\lambda \mapsto -\log \mathcal{L}_{\mathrm{YJ}}(\lambda)$ \eqref{eq:loglikelihoodYJ2} is strictly convex.
\end{proposition}
The proof of \Cref{th:convex} builds upon the work of \cite{kouider1995concavity} which shows that
the negative log-likelihood of the Box-Cox transformation \cite{box1964analysis} is convex.
The complete proof is deferred to \Cref{app:Proof}.
\paragraph{The exponential YJ algorithm}
The pseudo-code of the proposed algorithm is presented in \Cref{alg:dicYJ}, and relies on the exponential search presented in \Cref{alg:ExpUpdate} (cf \Cref{app:expsearch} for more details on exponential search). An illustration of \expYJ\ is shown in \Cref{illustrationSecYJ} in \Cref{app:expsearch}.
Due to the strict convexity of the negative log-likelihood of the YJ
transformation, we may perform the exponential search described in \Cref{sec:background} and \Cref{app:expsearch}. To do so, it is enough to obtain the sign of the derivative.
Let $\partial_\lambda \Psi(\lambda, \cdot)^2 = 2 \Psi(\lambda, \cdot) \partial_\lambda \Psi(\lambda, \cdot)$ and $\varphi(x) = \sign(x) + \log(|x|+1).$
Further, for~$g \in \{\Psi(\lambda, \cdot), \partial_\lambda \Psi(\lambda, \cdot),
\Psi(\lambda, \cdot)^2, \partial_\lambda \Psi(\lambda, \cdot)^2, \varphi\}$, let us define $ S_{g} \defeq \sum_{i=1}^{n} g(x_i).$
The derivative of the log-likelihood is available in closed form (see  \Cref{app:derivYJ}):
\begin{equation}
\label{eq:derivative}
\partial_\lambda \log \mathcal{L}_{\mathrm{YJ}} =  \frac{n}{2} \frac{S_{\partial \Psi^2} - 2(S_{\Psi} S_{\partial \Psi})/n}{S_{\Psi^2} - S^2_{\Psi}/n} - S_\varphi.
\end{equation}
Notice that $S_{\Psi^2} - S^2_{\Psi}/n$ can be expressed as a variance, hence is non-negative. We may therefore obtain $\sign\left[\partial_\lambda \log \mathcal{L}_{\mathrm{YJ}}\right]$ while avoiding performing division by computing
\begin{equation}\label{eq:signder2}
\sign\left[\partial_\lambda \log \mathcal{L}_{\mathrm{YJ}}\right]
 =  \sign\left[n S_{\partial \Psi^2} - 2S_{\Psi} S_{\partial \Psi} - 2S_\varphi(S_{\Psi^2} - S^2_{\Psi}/n)\right].
\end{equation}
Avoiding this division is crucial to make the overall procedure more numerical stable, as explained below, and eases the use of SMC routines.
\begin{figure}[h]
\begin{subfigure}{.47\textwidth}
	\centering
	\includegraphics[width=\textwidth]{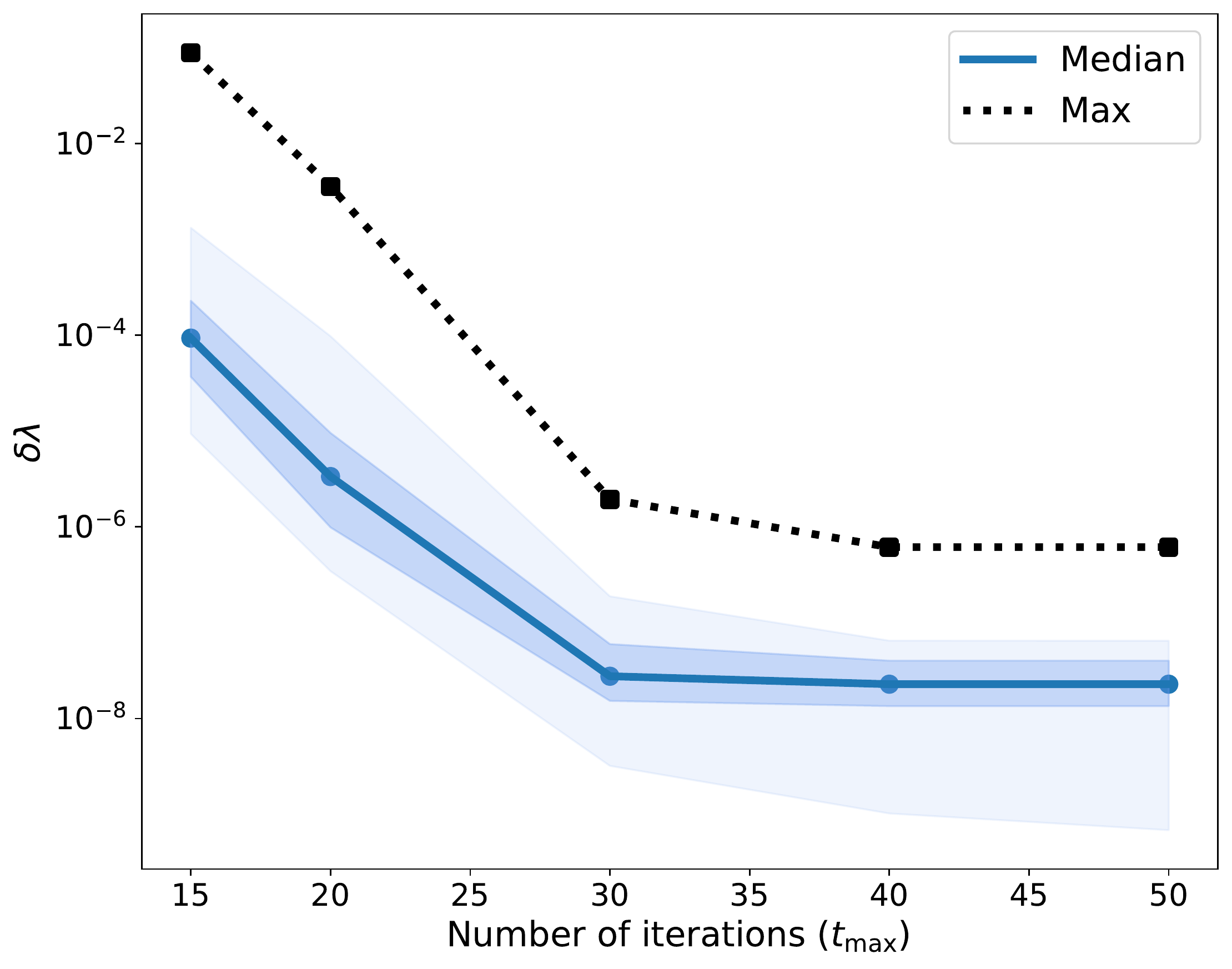}
\end{subfigure}
\begin{subfigure}{.49\textwidth}
     \centering
      \includegraphics[width=\textwidth]{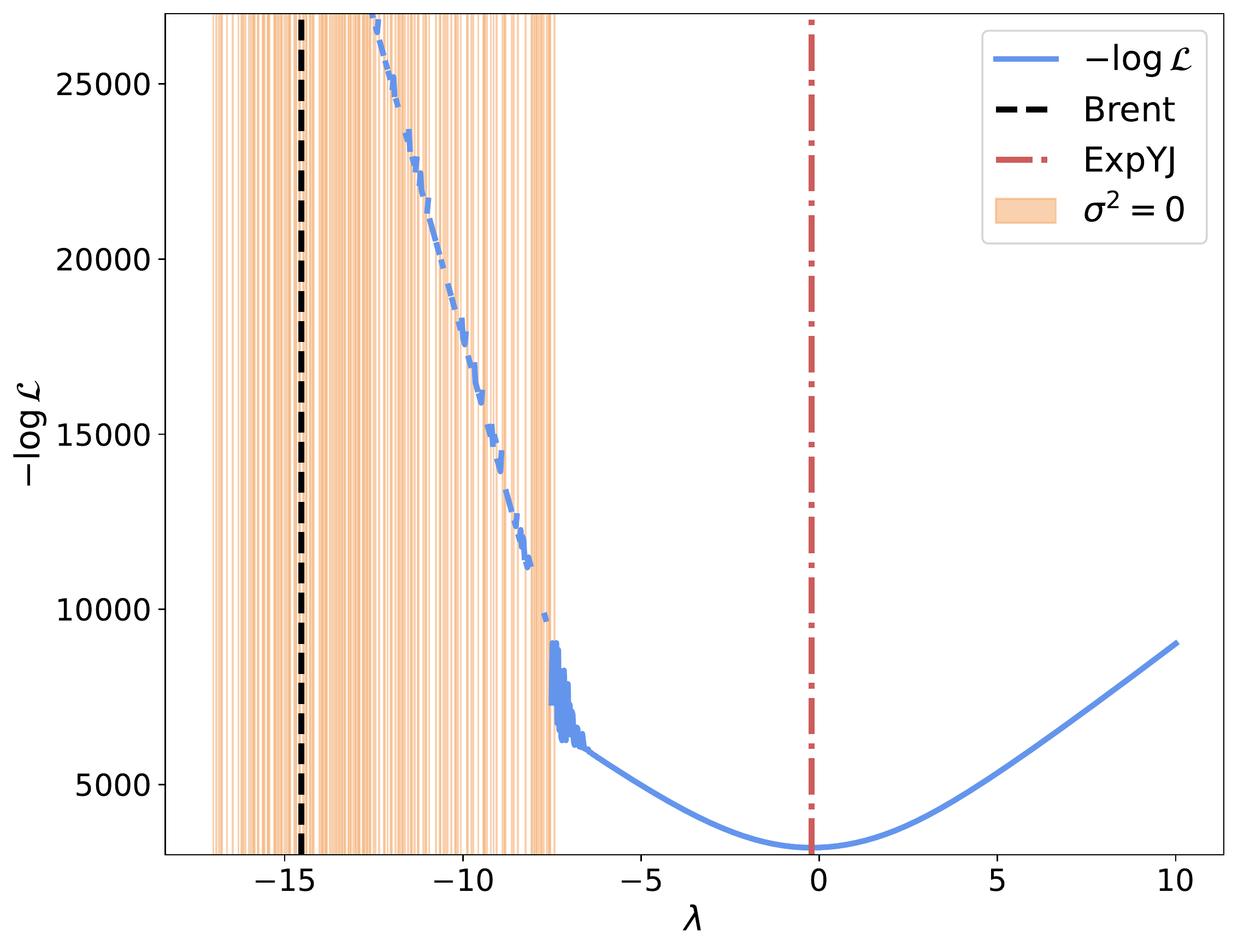}
\end{subfigure}
	\caption{Comparison of \expYJ\ and {\it scikit-learn}. {\bf Left}: For each of the 106 features (see \Cref{sec:datasets}), we
compute the relative difference  $\delta\lambda= \vert\lambda_{\textsc{ExpYJ}} -  \lambda_{\mathrm{sk}}\vert / \vert\lambda_{\mathrm{sk}}\vert$ and plot its median, maximum and $25\%$-$75\%$ and $10\%$-$90\%$ percentiles across the 106 features.
{\bf Right}:  Negative log-likehood of the YJ transformation 
for the mean area of the cell of each
sample of the {\it Breast Cancer} dataset. Full orange bars correspond to values of $\lambda$ for which the likelihood computed using scikit-learn returns $\infty$ as $\sigma^2_\lambda(\{x_i\})$ is equal to $0$ up to float-64 machine precision. Dotted lines correspond to the $\lambda_*$ found using Brent minimization or \expYJ\ with one client.}\label{fig:brent_exp_comparison}
\end{figure}
\begin{wrapfigure}{R}{.35\textwidth}
\vskip-.85cm
\begin{minipage}{0.35\textwidth} 
   \begin{algorithm}[H]
      \caption{\textsc{ExpUpdate}}
      \label{alg:ExpUpdate}
   \begin{algorithmic}
      \REQUIRE $\lambda$, $\lambda^+$, $\lambda^-$, $\Delta \in \{-1, 1\}$
       \IF{$\Delta = 1$}
         \STATE $\lambda^- \leftarrow \lambda$
         \STATE $\lambda \leftarrow (\lambda^+ + \lambda) / 2$ \textbf{if} $\lambda^+ < \infty$ \textbf{else} $\lambda \leftarrow \max(2\lambda, 1)$
         \ELSE
         \STATE $\lambda^+ \leftarrow \lambda$
         \STATE $\lambda \leftarrow (\lambda^- + \lambda) / 2$ \textbf{if} $\lambda^- > -\infty$ \textbf{else} $\lambda \leftarrow \min(2\lambda, -1)$
       \ENDIF 
       \ENSURE Updated $\lambda, \lambda^+, \lambda^-$
   \end{algorithmic}
   \end{algorithm}
\end{minipage}
\vskip-.5cm
\end{wrapfigure}

\paragraph{Accuracy of \expYJ}
We check the accuracy of \expYJ\ on the datasets
presented in \Cref{sec:datasets}. In particular, we compare the results provided by \expYJ\ with
the outputs of the {\it scikit-learn} algorithm based on Brent minimization.

For 2 of the 108 features present in the datasets, the {\it scikit-learn} implementation leads to numerical instabilities discussed hereafter.
Therefore, we focus our comparison on the 106 remaining features, that we aggregated regardless of the dataset. \Cref{fig:brent_exp_comparison}  reports the relative difference $\delta \lambda$ between the results obtained by \expYJ\ and
by the {\it scikit-learn} implementation as a function of the number of
iteration $t_\mathrm{max}$ (as defined in \Cref{alg:dicYJ}). These results show that this relative difference is of order less than ~$10^{-6}$ when $t_\mathrm{max} = 40$.
\paragraph{Numerical stability of \expYJ}
Our experiments demonstrate that \expYJ\ is numerically more stable than Brent minimization.
Indeed, for some values of $\lambda$ and some datasets $\{x_i\}$, the
transformation~$\Psi(\lambda, \cdot)$ concentrates all data points in a small interval such
that the values of $\Psi(\lambda, \{x_i\})$ are all equal up to machine precision.
In that case, the log-likelihood is not well-defined and the term $\log \sigma^2_{\Psi_\lambda}$
takes the value~$-\infty$, which prevents Brent minimization from converging.
This phenomenon does not affect the \expYJ\ routine as we do not compute directly the sign
of $\partial_\lambda \mathcal{L} = \partial_\lambda \sigma^2_{\Psi_\lambda}/\sigma^2_{\Psi_\lambda}- \sum_i \varphi(x_i)$,
but rather the sign of $\sigma^2_{\Psi_\lambda} \partial_\lambda \mathcal{L} = -\partial_\lambda \sigma^2_{\Psi_\lambda} - \sigma^2_{\Psi_\lambda} \sum_i \varphi(x_i)$, see Eq.~\eqref{eq:signder2}.

\Cref{fig:brent_exp_comparison} illustrates this in the case of a feature of the {\it Breast Cancer Dataset}.
The $\lambda_*$ returned by the Brent minimization method
of {\it scikit-learn} is $-14.53$ while the minimizer of the negative log-likelihood found by
the \expYJ\ is~$-0.21$. In particular, \Cref{fig:brent_exp_comparison} shows the values of the negative
log-likelihood as a function of $\lambda$ computed using 64-bit float precision.
The orange vertical full bands correspond to values for which $\sigma_{\Psi_\lambda}^2$ is zero within the
machine precision, resulting to a negative log-likelihood of $\infty$.
This instability happens for 2 of the 108 features used in numerical experiments, where blindly applying the Brent-based YJ transformation leads to all data points collapsing to zero, while \expYJ\ succeeds in transforming the data distributions to more Gaussian-like ones. \Cref{app:Brent} shows that this issue also arises in other real-life datasets.
\section{Applying the Yeo-Johnson transformation in FL}\label{sec:secYJ}

So far, we only considered the centralized setting, where data is accessible from a single server.
Yet, as mentioned in \Cref{sec:introduction,sec:background}, many real-world situations require working
with heterogeneous data split between different centers, and to take privacy constraints into account.
When the data is split across centers $k=1, \dots, K$ and the function to optimize is separable,
i.e.\ of the form $\mathcal{F}(\lambda) = \sum_{k=1}^K f_k(\lambda)$ where each $f_k$ can be computed
from data present in the center $k$ exclusively, Federated Learning techniques were recently proposed.
In short, they consist in repeatedly performing a few rounds of local optimization in each center,
before aggregating local parameters in the server. We refer to~\cite{kairouz2019advances} for an overview
of recent advances in FL.
In our case, however, the YJ negative log-likelihood \eqref{eq:loglikelihoodYJ2} is not separable,
due to the log-variance term. Indeed, turning the variance into a separable term would require
sharing the global YJ mean~$\mu_\lambda$ to all centers at each iteration.
Compared to the method we propose in this section, this would lead to more privacy leakage.

We now introduce \secYJ, a secure federated algorithm that builds upon \expYJ\ to apply YJ transformations.
This algorithm satisfies the two following properties: (i) it is \textit{pooled-equivalent}, i.e.\ it yields the same results as if the data were freely accessible from a single server, and (ii) it leaks as little information as possible about the underlying datasets, as shown by \Cref{prop:leakage}. 
Finally, it converges in a limited number of iterations,
thanks to the linear convergence of the underlying exponential search.

\paragraph{\secYJ}
 \secYJ\ is a federated adaptation of \expYJ\ presented
in \Cref{sec:expYJ}  to find the best
parameters $(\lambda_*, \mu_*, \sigma_*^2)$ of the YJ transformation when training datasets are split across different clients.
It relies on SMC to ensure that only the final
triplet~$(\lambda_*, \mu_*, \sigma_*^2)$ fitted on the training datasets is revealed,
without leaking any other information apart from the overall total
number of training samples $n$. Indeed, at each intermediate step, only the sign
of $\partial_\lambda \log \mathcal{L}_{\mathrm{YJ}}$ is revealed, and
the mean and variance of the transformed data is only revealed at the last step.
The pseudo-code of the resulting algorithm is presented in \Cref{alg:SecureFedYJ},
and relies on the exponential search presented in \Cref{alg:ExpUpdate}.
A functional representation of \secYJ\ is displayed \Cref{schemaYJ}.

In  \Cref{alg:SecureFedYJ},we label the clients by $k=1, \dots, K$ and each client $k$ holds
data $\{x_{k,i} :  i = 1, \dots, n_k\}$. We suppose that the total number of
samples $n = \sum_{k=1}^K n_k$ is public and shared to all clients. For a given function $g$, we denote $S_{k,g}$ the sum
$S_{k,g} \defeq \sum_{i=1}^{n_k} g(x_{k,i}).$
As introduced in \Cref{sec:background}, we use double brackets $\llbracket \cdot \rrbracket$ to indicate an SMC secret shared across the clients (see \Cref{app:SMC} for more details).
\begin{algorithm}[ht]
   \caption{\secYJ}
   \label{alg:SecureFedYJ}
\begin{algorithmic}
   \REQUIRE Data $\lbrace x_{k,i} \rbrace$, total data size $n$, number of steps $t_\mathrm{max}$
   \STATE {\bfseries Notations:} $\llbracket \cdot \rrbracket$ indicates a SMC secret shared across the clients. Any operation such as $\llbracket \cdot \rrbracket = f(\llbracket \cdot \rrbracket, \llbracket \cdot \rrbracket, \cdots)$ where $f$ can be the sum, product, or the sign, designs an SMC routine across the clients as described in \Cref{app:fdAgo}.
   \STATE Initialize $\lambda_{t=0} \leftarrow 0$, $\lambda^+ \leftarrow \infty$, $\lambda^- \leftarrow -\infty$ independently on each client
   \STATE Clients compute in SMC $\llbracket S_\varphi \rrbracket = \sum_k \llbracket S_{k, \varphi}\rrbracket$
   \FOR{$t=1$ {\bfseries to} $t_\mathrm{max}$}
   \FOR{$g \in \{\Psi(\lambda, \cdot),\Psi(\lambda, \cdot)^2, \partial_\lambda \Psi(\lambda, \cdot),  \partial\Psi(\lambda, \cdot)^2\}$}
   \STATE Clients compute in SMC $\llbracket S_g\rrbracket = \sum_k \llbracket S_{k,g}\rrbracket$,
      \ENDFOR
   \STATE Clients compute in SMC $\llbracket\Delta_t\rrbracket =  \sign\left[n \llbracket S_{\partial \Psi^2}\rrbracket\right. - 2\llbracket S_{\Psi}\rrbracket \llbracket S_{\partial \Psi} \rrbracket   \left. - 2\llbracket S_\varphi \rrbracket(\llbracket S_{\Psi^2}\rrbracket - \llbracket S_{\Psi} \rrbracket ^2/n)\right]$
   \STATE Clients reveal $\Delta_t$
   \STATE $\lambda_t, \lambda^-_t,\lambda^+_t \leftarrow \mathrm{\textsc{ExpUpdate}}(\lambda_{t-1}, \lambda^-_{t-1},\lambda^+_{t-1}, \Delta_{t})$ independently on each client
   \ENDFOR
   \STATE $\lambda_* \leftarrow \lambda_{t_\mathrm{max}}$
   \STATE Clients compute in SMC $\llbracket\mu\rrbracket =  \sum_k \llbracket S_{k,\Psi}\rrbracket / n$ and  $\llbracket\sigma^2\rrbracket =  \sum_k \llbracket S_{k,\Psi^2}\rrbracket / n - \llbracket\mu^2\rrbracket$
   \STATE Clients reveal $\mu_* \leftarrow \mu$ and $\sigma^2_* \leftarrow\sigma^2$
    \ENSURE The fitted triplet ($\lambda_*, \mu_*, \sigma_*^2)$
\end{algorithmic}
\end{algorithm}
\paragraph{Privacy leakage}
\label{sec:privacyleakage}
In \Cref{prop:leakage} we show that \Cref{alg:SecureFedYJ} only reveals information already contained in the fitted triplet $(\lambda_*, \mu_*, \sigma_*^2)$.
In comparison, turning the YJ negative log-likelihood \eqref{eq:loglikelihoodYJ2} into a log-separable function before using off-the-shelf FL methods would require sharing $\mu$ and its gradient and centrally computing $\sigma^2$ for intermediate values of $\lambda$ at each iteration. This could potentially lead to uncontrolled privacy leakage.
\begin{figure}
   \begin{subfigure}[b]{.48\textwidth}
      \centering
     \includegraphics[width=\textwidth]{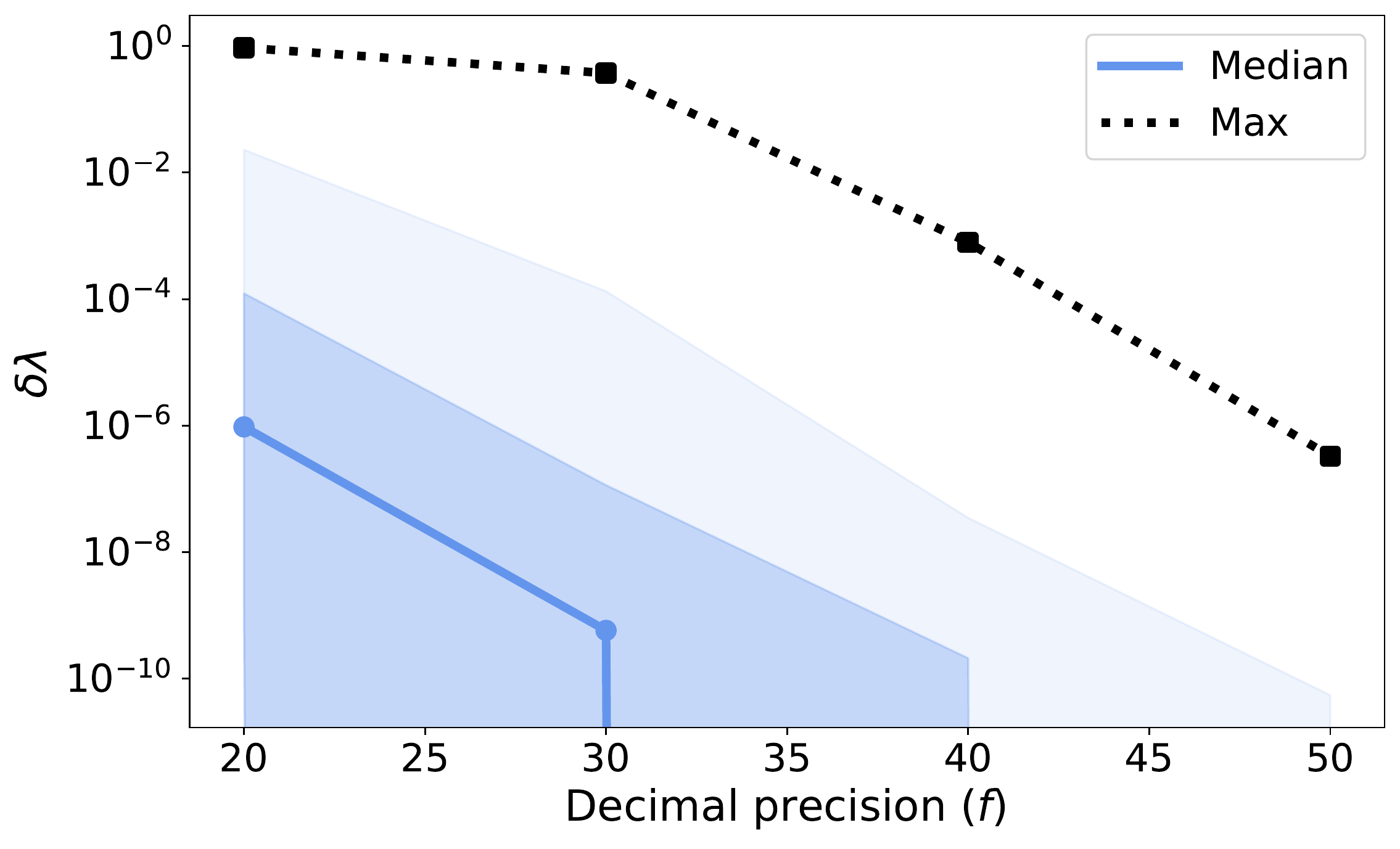}
       \vskip-.1cm
     \caption{Homogeneous}
  \end{subfigure}
  \begin{subfigure}[b]{.48\textwidth}
     \centering
     \includegraphics[width=\textwidth]{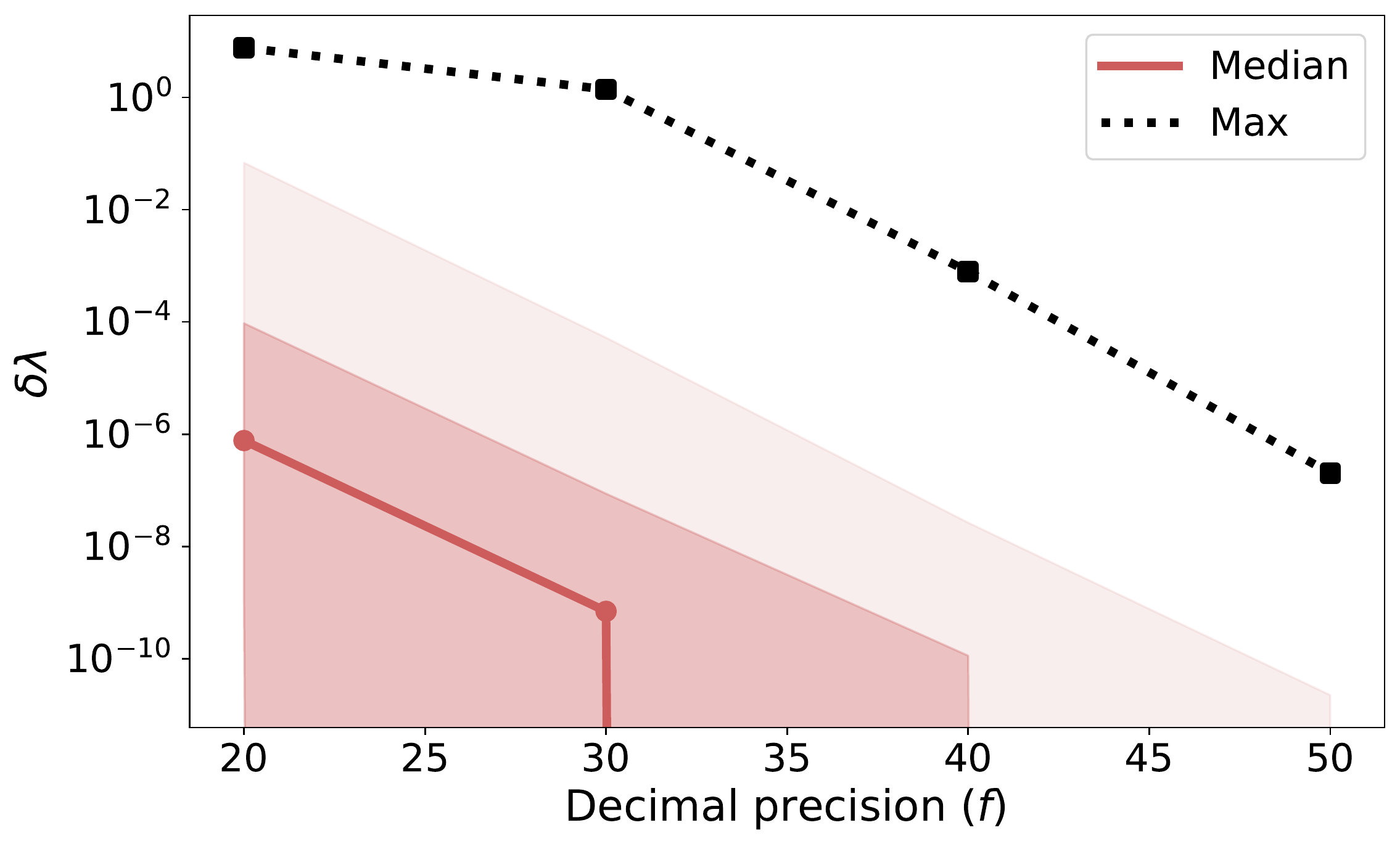}
       \vskip-.1cm
     \caption{Heterogeneous}
  \end{subfigure}
  \caption{Comparison of \secYJ\ and \expYJ\ for various fixed-point floating precisions $f$ used in SMC, with~$l=f+50$ and $t_\mathrm{max} = 40$. The data is distributed across 10 clients, either randomly (homogeneous, {\em left}), or per decile (heterogeneous, {\em right}, i.e.\ each client gets one decile of the data). We report the maximum, median, $25\%$-$75\%$ and $10\%$-$90\%$ percentiles of the relative error~$\delta\lambda= |\lambda_{\textsc{SecFedYJ}} - \lambda_{\textsc{ExpYJ}}| / |\lambda_{\textsc{ExpYJ}}|$
across the~108 features described in \Cref{sec:datasets}.}
\label{fig:hetero_homo_comparison}
\end{figure}
\begin{proposition}
\label{prop:leakage}
The fitted parameter $\lambda_*$ contains all the information revealed during the intermediate steps of \secYJ.
More precisely,  there exists a deterministic function $\mathcal{F}$ such that for any set of datasets
$\{x_{k,i}\}$ , if $\lambda_*(\{x_{k,i}\})$ is the result of \secYJ\ on $\{x_{k,i}\}$, then~$\{\lambda_t, \lambda^+_t, \lambda^-_t, \Delta_t\}_{t=1,\cdots, t_\mathrm{max}} = \mathcal{F}\left[\lambda_*(\{x_{k,i}\})\right]$
\end{proposition}
\textbf{Proof.}
This proposition comes from the fact that all gradient signs $\Delta_t$ revealed during the algorithm can be retrospectively
inferred from $\lambda_*$. Indeed,~$\partial_\lambda \log \mathcal{L}_{\mathrm{YJ}} < 0$
for $\lambda > \lambda_*$ and $\partial_\lambda \log \mathcal{L}_{\mathrm{YJ}} > 0$
for~$\lambda < \lambda_*$. Besides, the successive values of $\lambda_t$ explored at each step $t$ can be deterministically inferred from the initial value $\lambda_{t=0}$ and and the final fitted value $\lambda_*$. We construct such a function $\mathcal{F}$ and numerical verify this proposition in \Cref{app:AppF}. 
$\blacksquare$
\paragraph{Performance of \secYJ}
\label{subsec:sanitycheck}
We implement \secYJ\ in Python, using the MPyC library \cite{schoenmakers2018mpyc} based on Shamir Secret Sharing \cite{shamir1979share}. We refer to \Cref{app:SMC} for more details on our implementation. To represent signed real-valued numbers in an SMC protocol, we use a fixed-point representation (see \Cref{app:SMCFPR}) using $l$ bits, among which $f$ bits are used for the decimal parts. This means that we consider
floats ranging from~$-2^{l-f}$ to~$2^{l-f}$ and that we have an absolution precision of $2^{-f}$ in our computations.

In order to ensure the accuracy of \secYJ\ results, we need to make sure that $l$ and $f$ are large enough. \Cref{fig:hetero_homo_comparison} shows the accuracy of \secYJ\ when compared to \expYJ\ for various values of~$f$. According to these numerical experiments, taking~$f = 50$ and~$l=100$ provides reasonably accurate results. Moreover, by construction, the outputs of \secYJ\ do not depend on how the data is split across the clients, up to rounding numerical errors. Therefore this algorithm is resilient to data heterogeneity, as long as the numerical decimal precision $f$ is large enough, as shown in \Cref{fig:hetero_homo_comparison}.

Performing \secYJ\ with~$t_\mathrm{\max} = 40$ takes 726 rounds of communication (see \Cref{app:SMCrounds}). During these communication rounds, each client sends overall about 8~Mb per feature to every other client (see also \Cref{app:SMCrounds}). \secYJ\ can be applied independently and in parallel to each feature. Therefore, the overall number of rounds does not depend on the number of features being considered, and the communication costs grow proportionally to the number of features.
In a realistic cross-silo FL setting as described in \cite{fro2021nature}, the bandwidth of the network is $1$~Gb per second with a delay of~$20\ \mathrm{ms}$ 
between every two clients. In this context, the execution of \secYJ\ with~$t_\mathrm{\max} = 40$ on $p$ features would take about $726\times 20\ \mathrm{ms} \simeq 15\ \mathrm{s}$ due to the communication overhead, in addition to $p\times8\ \mathrm{Mb} / 1\ \mathrm{Gbps} \simeq 8 p\ \mathrm{ms}$ due to the bandwidth. This shows that \secYJ\ is indeed a viable algorithm in a real-world scenario.

As pointed out in \Cref{app:expsearch}, the binary search in the exponential search can be replaced by a $k$-ary search. In such a setting, the sign of the negative log-likelihood of the YJ transformation is computed for $k-1$ different values of $\lambda$ at each round. Such a modification would reduce the number of communication rounds required to obtain a given accuracy, while increasing the size of the data exchanged over the network at each round.
\section{Applications}
\label{sec:applications}

\paragraph{Genomic data: TCGA}

\begin{wrapfigure}{r}{.55\textwidth}
 \vskip-.5cm
   \centering
     \includegraphics[width=.55\textwidth]{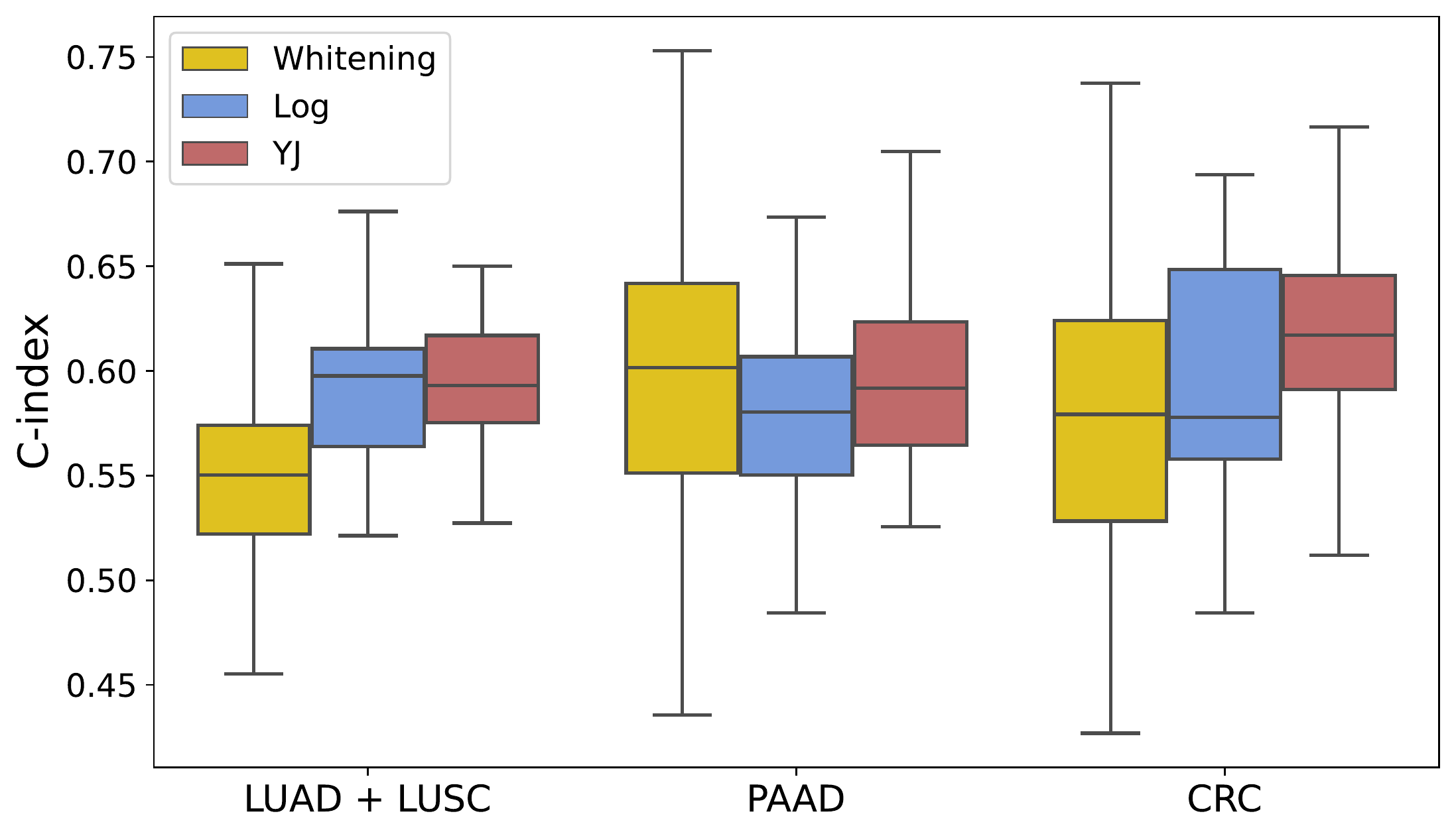}
      \caption{Cross-validation survival analysis performance (higher is better) of
      a CoxPH model with different normalization methods. The YJ transformation yields
      either a better or on par performance, and further stabilizes results compared
      to other approaches.}
     \label{fig:tcga_exp}
     \vskip-.3cm
\end{wrapfigure}

We start by showing the benefits of YJ preprocessing in survival analysis experiments
on lung (LUAD+LUSC), pancreas (PAAD), and colorectal (CRC) cancers.
Given gene expression raw counts (features) and censored survival data (responses)
from patients having either of those three cancers, we aim to fit a
Cox Proportional Hazards (CoxPH) model~\cite{cox1972regression} with the highest possible concordance index
(C-index)~\cite{kleinbaum2010survival},
which measures how well patients are ranked with respect to their survival times.
We refer to~\cite{kleinbaum2010survival} for a more thorough introduction to survival analysis.
In \Cref{fig:tcga_exp}, we compare three different preprocessing methods: (i) whitening, (ii) log normalization, and (iii) YJ, each followed by a PCA dimensionality reduction step.
More precisely, whitening (i) consists in centering and reducing to unit variance the total read counts of all genes across all samples,
log normalization consists in applying $u \mapsto \log(1+u)$ to raw read counts before applying global whitening,
and YJ is a global YJ transform on the total read counts.
We then evaluate each strategy using 5-fold cross-validation with 5 different seeds.
We refer to \Cref{subsec:tcga_hyperparams} for experimental details.
While this experiment is performed in a pooled environment,
note that, importantly, each step has a federated pooled-equivalent version: apart from
the proposed \secYJ~for YJ, see e.g. \cite{grammenos2020fed} for PCA, and
Webdisco~\cite{lu2015webdisco} for Cox model fitting.
This simplified setting allows us to understand the importance 
of the Yeo-Johnson transformation in an ideal setting, independently of other potential
downstream federated learning artifacts.

In \Cref{fig:tcga_exp}, we see that YJ is better or on par with the best method for each cancer:
YJ improves prediction results for colorectal cancer, while
yielding results which are on par with the best results for lung and pancreas cancers,
with a smaller variance.

\paragraph{Synthetic data}
\begin{figure}[ht]
   \centering
   \begin{subfigure}[t]{.47\textwidth}
     \includegraphics[width=\textwidth]{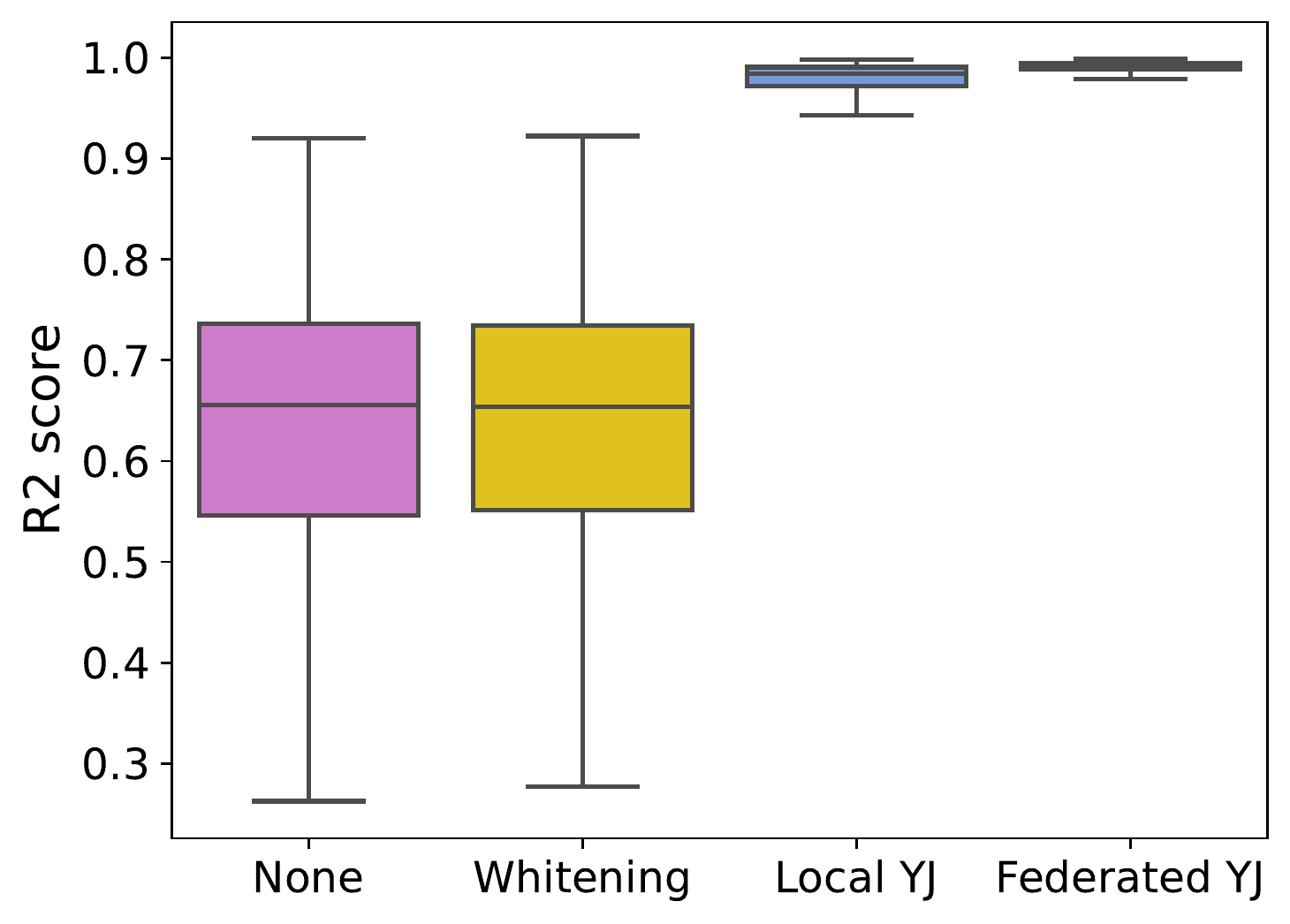}
     \label{subfig:mock_1}
    \end{subfigure}
    \begin{subfigure}[t]{.48\textwidth}
    \raisebox{-.3cm}{
     \includegraphics[width=\textwidth]{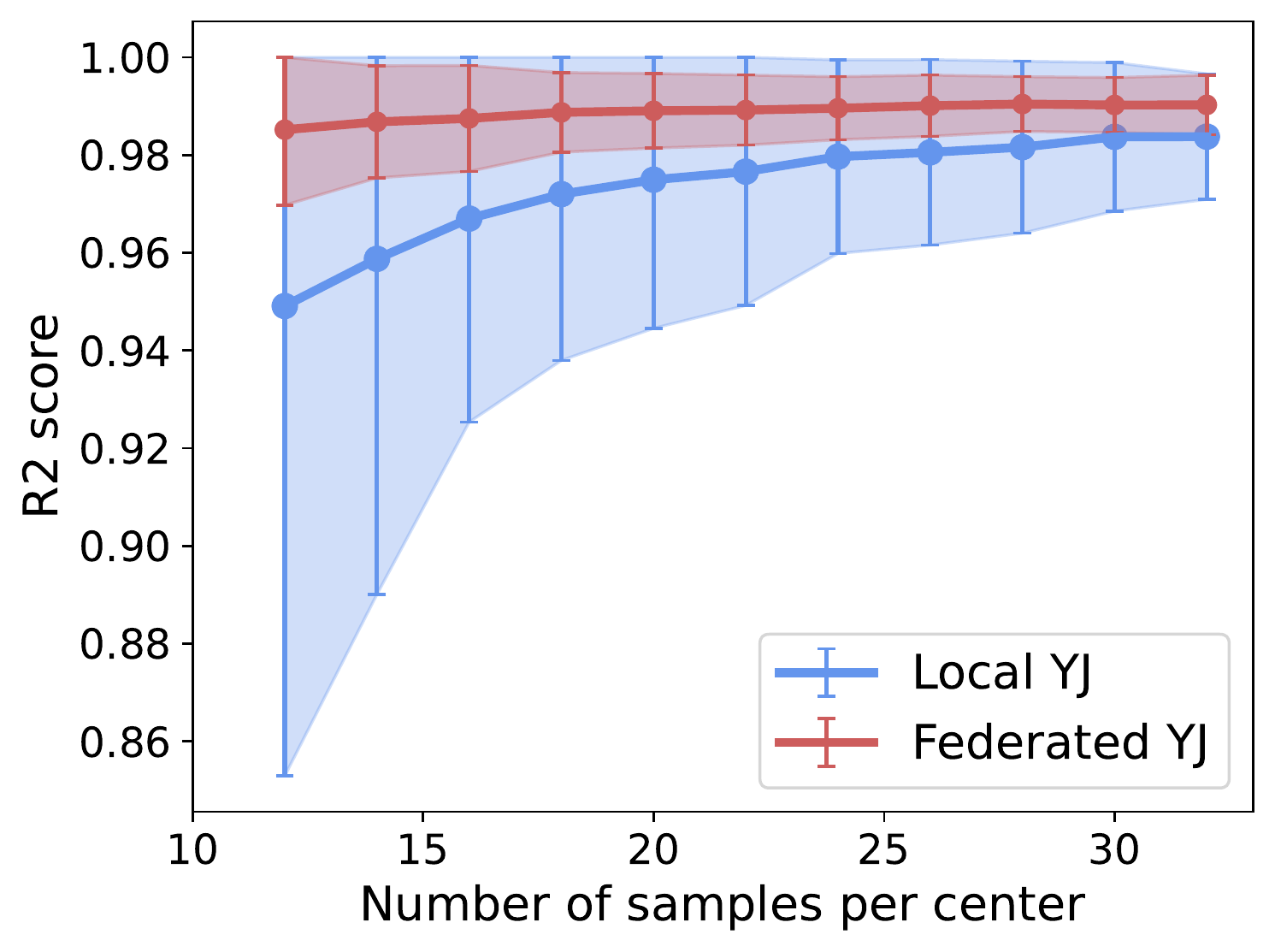}
    }
         \label{subfig:mock_2}
    \end{subfigure}
     \caption{Comparison of different preprocessing methods for linear regression on synthetic federated data. {\bf Left}: performance with 200 samples (20 on each of the 10 centers)  on 1000 independent draws, showing the interest of using YJ preprocessing. The R2 score of the models are computed on another dataset of 200 samples not seen during the training. {\bf Right}: Comparison of local and federated YJ over 1000 independent draws. In local YJ, a single center is randomly chosen to fit $\lambda$, which is then shared with other centers, to ensure that the same transformation is applied everywhere. Full lines correspond to the means, error bars to $\pm$ std of the R2 scores of the model on an unseen test dataset.}
     \label{fig:mock_data}
\end{figure}
We show how applying YJ may help improving performance compared to no or basic preprocessing, and how \secYJ\ yields improvements compared to local YJ transforms in federated linear regression.
To do so, we generate covariates $\widetilde{X}$ and responses $y$ as
\begin{align}\label{eq:covariates}
\begin{split}
	\widetilde{X} & = (\exp(x_1), \exp(x_2 + 2), \sigma(x_3)) \text{ with } X = (x_1, x_2, x_3) \sim \mathcal{N}(0, \mathbf{I}_3), \\
	y& = \beta^T X + \varepsilon \text{ with } \varepsilon \sim \mathcal{N}(0, 0.1),
\end{split}
\end{align} 
where $\sigma(\cdot)$ is the sigmoid function and $\beta = (-1.3, 2.4, 0.87)$ was randomly chosen.
The goal is then to fit a linear model from i.i.d. samples $(\widetilde{X}_i, y_i), i=1, \dots, n$ following \eqref{eq:covariates}, after an optional preprocessing step. Similarly to the previous example, we simulate a cross-device FL setting only for the preprocessing steps, and the linear model is then fitted in a pooled setting for simplicity. All the details of the numerical experiments are provided in \Cref{app:hyperparams}.
We suppose that the samples~$(\widetilde{X}_i, y_i), i=1, \dots, n$ are homogeneously split across $10$ centers. The responses $y_i$ have a highly nonlinear dependency on the covariates $\widetilde{X}_i$, but depend linearly on the $X_i$'s (up to Gaussian noise), which are not observed. 
Hence, we expect that applying a suitable preprocessing step before training a linear model will transform back the $\widetilde{X}_i$'s into the normally-distributed $X_i$'s and lead to a high performance, compared to no transformation.
The results of our experiments are summarized in \Cref{fig:mock_data}. The left figure shows that the YJ transformation is indeed capable of roughly inverting the~$\widetilde{X}_i$'s into the~$X_i$'s, yielding a major improvement compared to no preprocessing or standard centering and reduction to unit variance. Besides, the right figure shows that even in this homogeneous setting where the data is i.i.d.\ across centers, using a federated version of YJ compared to a local version of YJ leads to better average performance, and reduced variance.

\section{Conclusion}
\label{sec:conclusion}
\paragraph{Summary of our contributions}
In this work, we introduce \secYJ, a method to fit a YJ transformation on
data shared by different clients in a cross-silo setting. \secYJ\ is an SMC version of its pooled equivalent \expYJ\, which builds upon the convexity of the negative
log-likelihood of the YJ transformation, a novel result introduced by this work,
and on the fact that the sign of its derivative can be
computed in a stable way. We show that \secYJ\ has
the same accuracy as a standard YJ transformation
on pooled data. In particular, the results do not depend on how the data is split across the clients,
making \secYJ\ resilient to data heterogeneity.
Besides, the quantities disclosed by \secYJ\ during
the training to the central server do not leak any other information than what is
contained in the final parameters $(\mu_*, \lambda_*, \sigma^2_*)$.

\paragraph{Limitations and future work}
While Brent minimization has a
super-linear convergence, our approach only has a linear convergence,
as it relies on exponential search. This can be an issue
if the communication costs between the clients and the server are high. Acceleration could be achieved
by either adapting Brent minimization to a cross-silo setting, or applying a second-order method. We leave the development of a faster SMC methods using either of those two approaches to future work.

Another limitation is that even if our approach reveals only information
that would be contained in the final fitted parameters, such parameters
themselves might leak information about individual samples, as our approach is not differentially private (DP)~\cite{dwork2006differential}.
By adding Gaussian or Laplacian noise to each sample's features when computing the~$S_g$ terms
one could, in principle, make the resulting algorithm DP~\cite{abadi2016deep}. However it is unclear 
to what extent the noise would impact the final accuracy of the method.

Finally, we only consider an {\it honest-but-curious} setting.
We do not explore the threat of a malicious participant
that would purposely deviate from the protocol to either gain more information or to jeopardize the convergence.
We leave this investigation to future work.

\section*{Acknowledgement}
The authors would like to thank the four anonymous reviewers, as well as the anonymous area chair reviewer for their relevant comments and ideas which significantly improved the paper.

\bibliography{bibli}
\bibliographystyle{plainnat}

\section*{Checklist}

\begin{enumerate}

\item For all authors...
\begin{enumerate}
  \item Do the main claims made in the abstract and introduction accurately reflect the paper's contributions and scope?
    \answerYes{}
  \item Did you describe the limitations of your work?
    \answerYes{We described them in \Cref{sec:conclusion}}
  \item Did you discuss any potential negative societal impacts of your work?
    \answerNA{}
  \item Have you read the ethics review guidelines and ensured that your paper conforms to them?
    \answerYes
\end{enumerate}

\item If you are including theoretical results...
\begin{enumerate}
  \item Did you state the full set of assumptions of all theoretical results?
    \answerYes
        \item Did you include complete proofs of all theoretical results?
    \answerYes{The full proof of the main theoretical results, i.e. the convexity of the Yeo-Johnson negative log-likelihood (\Cref{th:convex}) is provided in \Cref{app:Proof}.}
\end{enumerate}

\item If you ran experiments...
\begin{enumerate}
  \item Did you include the code, data, and instructions needed to reproduce the main experimental results (either in the supplemental material or as a URL)?
    \answerNo{The code of the experiments is not provided, but a detailed pseudo-code of the newly proposed algorithms are provided.}
  \item Did you specify all the training details (e.g., data splits, hyperparameters, how they were chosen)?
    \answerYes{We specified all the hyperparameters and the details of the numerical experiment in \Cref{app:NR}.}
        \item Did you report error bars (e.g., with respect to the random seed after running experiments multiple times)?
    \answerYes{Standard deviations or quantiles of the results with respect to the seed are provided (cf plots) }
        \item Did you include the total amount of compute and the type of resources used (e.g., type of GPUs, internal cluster, or cloud provider)?
    \answerNo{The experiment are not heavy and run easily on a personal computer, on a CPU}
\end{enumerate}

\item If you are using existing assets (e.g., code, data, models) or curating/releasing new assets...
\begin{enumerate}
  \item If your work uses existing assets, did you cite the creators?
    \answerYes{The datasets used are open datasets available online and are systematically cited.}
  \item Did you mention the license of the assets?
    \answerYes{The licence of the datasets used are provided in \Cref{sec:datasets}}
  \item Did you include any new assets either in the supplemental material or as a URL?
    \answerNo
  \item Did you discuss whether and how consent was obtained from people whose data you're using/curating?
    \answerNA{We are using open datasets available online. The genomic dataset from TCGA have been previously anonymised by its creator before publication}
  \item Did you discuss whether the data you are using/curating contains personally identifiable information or offensive content?
    \answerNA{}
\end{enumerate}

\item If you used crowdsourcing or conducted research with human subjects...
\begin{enumerate}
  \item Did you include the full text of instructions given to participants and screenshots, if applicable?
    \answerNA
  \item Did you describe any potential participant risks, with links to Institutional Review Board (IRB) approvals, if applicable?
    \answerNA
  \item Did you include the estimated hourly wage paid to participants and the total amount spent on participant compensation?
    \answerNA
\end{enumerate}

\end{enumerate}


\newpage

\appendix

\section{Additional properties of the Yeo-Johnson transformation}
\subsection{Derivation of the Yeo-Johnson log-likelihood}
\label{app:likelihood}
Using the change of variables rule, the probability to draw a set of points $\{x_i\}$ such that $\{\Psi(\lambda, x_i)\}$ follows a Gaussian distribution of mean $\mu$ and variance $\sigma^2$ is given by:
\begin{equation}
\mathbb{P}(\{x_i\} | \lambda, \mu, \sigma) = \mathbb{P}(\{\Psi_\lambda(\lambda, x_i)\} | \lambda, \mu, \sigma) * \det J [\{x_i\}, \Psi(\lambda, x_i)]
\label{eq:applikelihood1}
\end{equation}
where $\det J [x_i, \Psi(\lambda, x_i)]$ is the determinant of the Jacobian matrix $J[\{x_i\}, \{\Psi(\lambda, x_i)\}]$ defined as:

\begin{equation}
J\left[\{x_i\}, \{\Psi(\lambda, x_i)\}\right]_{ab} = \frac{\partial \Psi(\lambda, x_a)}{\partial x_b}
\end{equation}

This matrix is diagonal and each term of its diagonal can be computed using Eq.~\eqref{eq:YJanalytics}. For each value of $\lambda$ and $x_i$, these diagonal terms can be re-written as $\exp[(\lambda-1)\sign(x_i) \log (|x_i|+1)]$. The term $\mathbb{P}(\{\Psi_\lambda(\lambda, x_i)\} | \lambda, \mu, \sigma)$ is equal to:

\begin{equation}
\mathbb{P}(\{\Psi_\lambda(\lambda, x_i)\} | \lambda, \mu, \sigma) = \prod_{i} \frac{1}{\sigma \sqrt{2\pi}}\exp\left[- \frac{(x_i - \mu)^2}{2 \sigma^2}\right]
\end{equation}

By taking the logarithm of  Eq.~\eqref{eq:applikelihood1}, we obtain the log-likelihood provided in \Cref{sec:introduction} and originally derived on \cite{yeo2000new}.
\subsection{Relationship with the Box-Cox transformation}
\label{app:boxcox}
The Box-Cox transformation \cite{box1964analysis} works similarly to the YJ transformation, but only applies to strictly positive data.
The Box-Cox transformation is based on a function $\Phi(\lambda, \cdot)$ parametrized by $\lambda$ and defined for $x > 0$ as:
 \begin{align}
  \Phi(\lambda, x) = \begin{cases}
        \frac{x^\lambda-1}{\lambda}, & \text{if } \lambda \neq 0,\\
       \ln(x), & \text{if } \lambda = 0.\\
       \end{cases}
\end{align}
It is straightforward to check that for any $\lambda \in \mathbb{R}$ the YJ transformation $\Psi$ and the Box-Cox transformation $\Phi$ are related by the following equations:
\begin{equation}
\Psi(\lambda, x) = \begin{cases}
\Phi(\lambda,  x+1) & \text{if } x \geq 0,\\ 
- \Phi(2-\lambda, 1 - x) & \text{if } x  < 0.
\end{cases}
\end{equation}

\subsection{Analytical formulae for the derivatives of the Yeo-Johnson transformation}
\label{app:derivYJ}
The YJ function is infinitely differentiable with respect to
both of its variables ($x$ and $\lambda$). Here are its successive derivatives with respect to $\lambda$:
 \begin{equation}
        \label{eq:DerYJAnalytics}  
   \partial_\lambda^k  \Psi(\lambda, x) = \begin{cases}
        [(x+1)^\lambda [\ln(x+1)]^k -k \partial_\lambda^{k-1}\Psi(\lambda, x)]/ \lambda, & \text{if } x \geq 0, \lambda \neq 0,\\
        \ln(x+1)^{k+1} / (k+1), & \text{if} x \geq 0, \lambda = 0,\\
	([-x+1]^{2-\lambda} [\ln(-x+1)]^k + k \partial_\lambda^{k-1}\Psi(\lambda, x) )/ (2-\lambda), & \text{if } x < 0, \lambda \neq 2,\\
    (- \ln(-x+1))^{k+1} / (k+1), & \text{if} x < 0, \lambda = 2.\\
        \end{cases}
\end{equation}

\section{Background on exponential search}
\label{app:expsearch}
Exponential search~\cite{bentley1976almost} is a method to look for an element
in an unbounded sorted array. The idea is to first find bounds on the array such
that the element is contained within such bounds, and then perform a classic
binary search inside these bounds. Let us consider the task of finding the smallest
element~$u_{i_0}$ greater than a threshold $C$ in an unbounded sorted
array~$\{u_i\}_{i\in\mathbb{N}^*}$. The exponential search iteratively looks at $u_i$
for $i \in \{1, 2, 2^2, 2^4, \dots \}$ until it finds a $i_\mathrm{max}$ such
that~$u_{i_\mathrm{max}} \geq C$. This takes $\log_2 (i_\mathrm{max})$ steps.
Then it performs a binary search between $i=i_\mathrm{max}/2$ and $i=i_\mathrm{max}$, which also
takes $\log_2 (i_\mathrm{max} /2)$ steps.

If $f(s)$ is a strictly increasing function of $s$ taking both positive and negative values,
one can adapt the exponential search to find the root $s_0$ of $f$. The first
step is to find an upper and a lower bound of $s_0$ by evaluating $f$ at
different points using an exponential grid (e.g.\ evaluating~$f$
in~$s = 1, -1,  2, -2, 2^2, -2^2, 2^4, -2^4, \dots$). Once such bounds are found, one
can perform a dichotomic search inside these bounds to find the root $s_0$
of $f$. This dichotomic search has a linear convergence of order $2$,
with each step summarized in \Cref{alg:ExpUpdate}.
It is important to note that this algorithm is correct even if $f$ is not increasing,
as long
as~$f(s) < 0$ when~$s < s_0$ and~$f(s) > 0$ when~$s > s_0$, as is the
case in this work when $f$ is the derivative of the negative YJ log-likelihood.

\Cref{illustrationSecYJ} is an illustration of \expYJ\ that is based on exponential search.
\begin{figure}[!ht]
     \centering
       \includegraphics[width=0.48\textwidth]{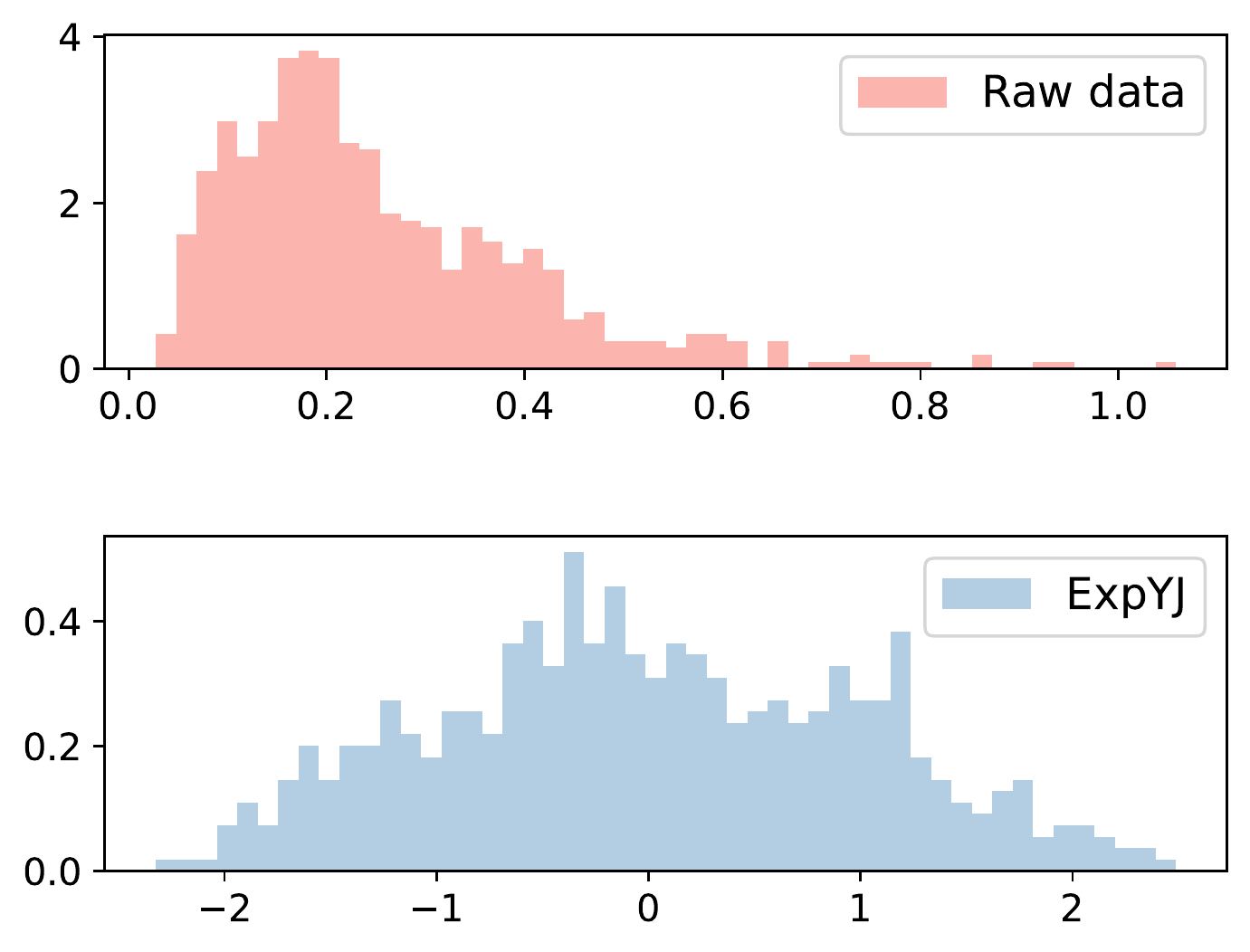}
      \includegraphics[width=0.48\textwidth]{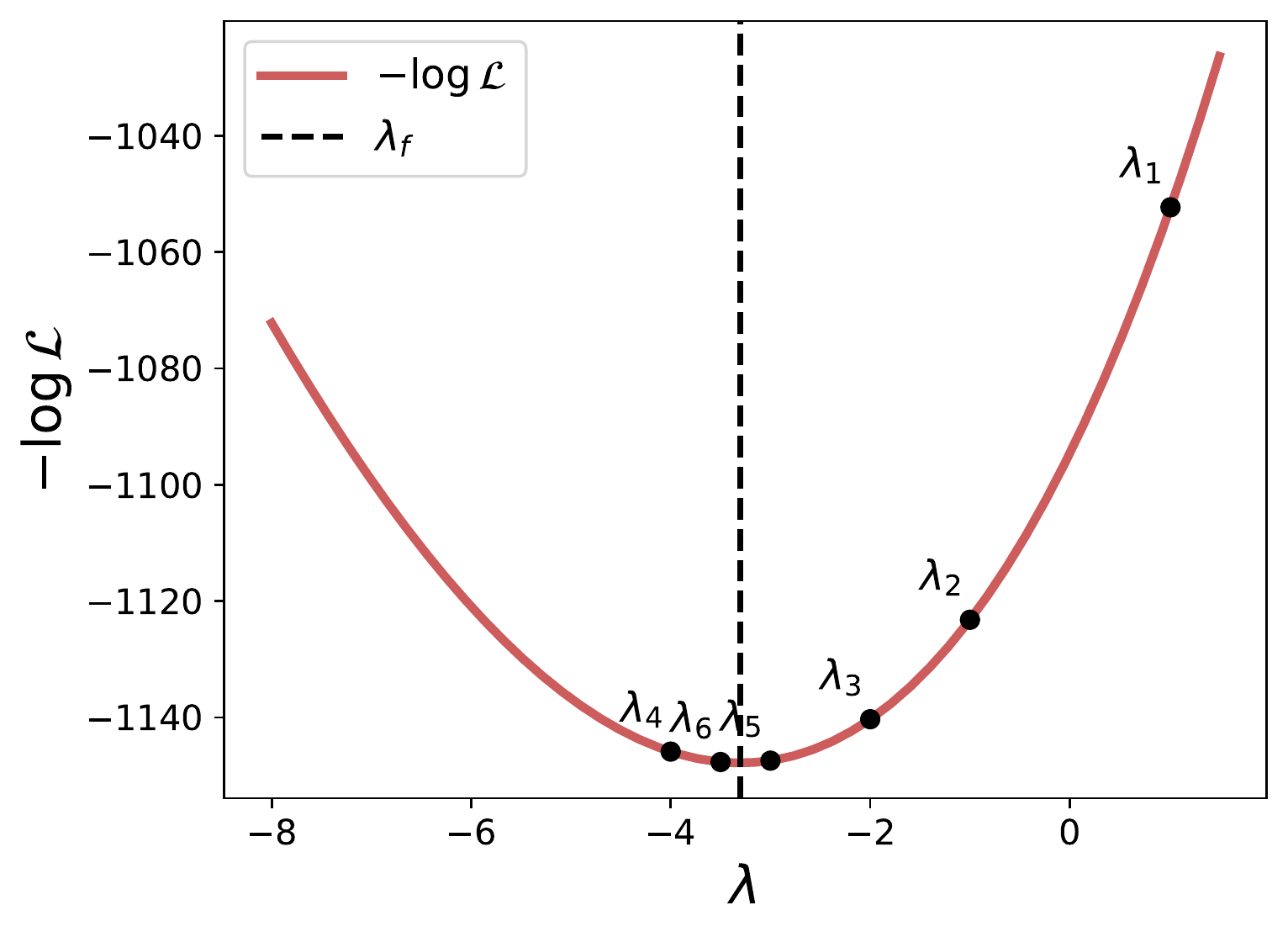}
\caption{An example of \expYJ\ applied to the largest perimeter of the cells in each sample of the {\it Breast Cancer} dataset. {\em Left:} histogram of the pooled dataset before ({\it top}) and after ({\it bottom}) applying a YJ transformation with fitted parameters.
{\em Right:} negative log-likelihood of the YJ transformation as a function
of $\lambda$. The points~$\lambda_t$ correspond to the values taken
by \expYJ\ during exponential search.}
\label{illustrationSecYJ}
\end{figure}

A natural extension of the exponential search is to replace the binary search into a $k$-ary search during the dichotomic search. In that case, $k-1$ values of $f$ are computed at each round. Such a modification reduces the number of steps required for a given accuracy while increasing the number of operations performed at each step.

\section{Proof of \Cref{th:convex}}
\label{app:Proof}
We first introduce some lemmas that will be required for the main proof.
\begin{lemma} \label{lemma:lnconvex}
Let $\lambda \mapsto f_i(\lambda)$, $i=1,\dots, I$ be positive
and twice differentiable functions, such that for
all $i$, $\lambda \mapsto \ln [f_i (\lambda)]$ is convex.
Then $\lambda \mapsto \ln[\sum_i f_i(\lambda)]$ is also convex.
\end{lemma}

\begin{proof}

The proof of \Cref{lemma:lnconvex} is based on the following lemma:
\begin{lemma}
\label{lemma:inequality}
Let $\{a_i\}_{i=1\dots I}, \{b_i\}_{i=1\dots I}, (c_i)_{i=1\dots I}$ be real numbers,
such that for all $ 1 \leq i \leq I$:  $a_i \geq 0$, $b_i \geq 0$ and $a_i  b_i \geq c_i^2$. Then
$$\left(\sum_{i=1}^I a_i\right) \left(\sum_{i=1}^I b_i\right) \geq \left(\sum_{i=1}^I c_i \right)^2.$$
\end{lemma}
Indeed, it holds,
\begin{equation}
\left(\sum_{i=1}^I a_i\right)\left(\sum_{i=1}^I b_i\right) - \left(\sum_{i=1}^I c_i \right)^2
= \left(\sum_{i=1}^I a_i b_i - c_i^2\right)  + \left(\sum_{i=1}^I \sum_{i<j\leq I}a_i b_j + a_j b_i - 2 c_i c_j \right).
\end{equation}
The first sum contains only non-negative terms as $\forall i,\ a_i  b_i \geq c_i^2$.
Recalling that $a_i \geq 0$ and $b_i \geq 0$,
the second sum also contains non-negative terms
as $a_i b_j + a_j b_i - 2 c_i c_j \geq a_i b_j + a_j b_i - 2 \sqrt{a_i b_i}\sqrt{a_j b_j} = (\sqrt{a_i b_j} - \sqrt{a_j b_i})^2 \geq 0$

Now let us prove \Cref{lemma:lnconvex}.
The convexity and twice differentiability of $\ln [f_i (\lambda)]$ implies
that~$\partial^2_\lambda \ln [f_i (\lambda)] \geq 0$ and therefore
that
\begin{align}
\label{eq:lemmaproofeq}
f_i \partial^2_\lambda f_i - (\partial_\lambda f_i)^2 \geq 0.
\end{align}
As $f_i > 0$, we can conclude from Eq.~\ref{eq:lemmaproofeq} that $\partial_\lambda^2 f_i \geq 0$.
Using \Cref{lemma:inequality} and the linearity of the derivative, we have:
\begin{align}
\left(\sum_i f_i \right) \partial^2_\lambda \left(\sum_i  f_i\right) -
\left( \partial_\lambda  \sum_i f_i\right)^2 \geq 0.
\end{align}
which means that $\partial^2_\lambda \ln (\sum_i f_i) \geq 0$.

\end{proof}

\begin{lemma} \label{lemma:addsigma}
Let $\lbrace \alpha_i \rbrace_{i = 1, \dots, n_\alpha}$ and $\lbrace \beta_i \rbrace_{i = 1, \dots n_\beta}$
be two non-empty sets of real numbers, and let us denote
$\lbrace \gamma_i \rbrace_i = \{\alpha_1,\dots, \alpha_{n_\alpha}, \beta_1, \dots, \beta_{n_\beta}\}$
and $n_\gamma = n_\alpha + n_\beta$.
Let $\bar \alpha = \frac{1}{n_\alpha} \sum_i \alpha_i$,  $\bar \beta = \frac{1}{n_\beta} \sum_i \beta_i$,
$\bar \gamma = \frac{1}{n_\gamma} \sum_i \gamma_i$
and $\sigma_\alpha^2 = \frac{1}{n_\alpha}\sum_i (\alpha_i - \bar \alpha)^2$,
$\sigma_\beta^2 = \frac{1}{n_\beta}\sum_i (\beta_i - \bar \beta)^2$,
$\sigma_\gamma^2 = \frac{1}{n_\gamma}\sum_i (\gamma_i - \bar \gamma)^2$. Then:
\begin{equation}
\sigma_\gamma^2 =  \frac{n_\alpha}{n_\gamma} \sigma_\alpha^2 + \frac{n_\beta}{n_\gamma} \sigma_\beta^2 + \frac{n_\alpha n_\beta}{n_\gamma^2} (\bar \alpha- \bar \beta)^2.
\end{equation}
\end{lemma}
\begin{proof}
This identity is easily obtained using the definitions
of $\sigma_\alpha^2, \sigma_\beta^2$ and $\sigma_\gamma^2$.
\end{proof}

\subsection{Proof of \Cref{th:convex}}

\begin{proof}

We start by proving the only shows the convexity of $-\log \mathcal{L}_{\mathrm{YJ}}(\lambda)$, and prove strict convexity in \Cref{app:strictconvexity}.

Let $\{x_i\}_{i=1\cdots n}$ be our data points and let us split this dataset
into non-negative values $\{x^+_i\} = \{x_i | x_i \geq 0\}$ and negative
values $\{x^-_i\} = \{x_i | x_i < 0\}$.
Let $\gamma_i = \Psi(\lambda, x_i)$, $\alpha_i = \Psi(\lambda, x_i^+)$,
and~$\beta_i = \Psi(\lambda, x^-_i)$. We denote $n_\alpha , n_\beta, n_\gamma$
the lengths of the sets $\{\alpha_i\}$, $\{\beta_i\}$ and 
$\{\gamma_i\}$. For clarity, let us consider the case where both $\{x_i^+\}$ and $\{x_i^-\}$
have at least two distinct items and
therefore~$n_\alpha \geq 1$, $n_\beta \geq 1$ and~$\sigma^2_\alpha > 0$, $\sigma^2_\beta > 0$.
We relegate to \Cref{app:edge} the other edge cases.
According to \Cref{lemma:addsigma}, the expression of negative log-likelihood of
the YJ transformation provided in Eq.~\eqref{eq:loglikelihoodYJ2} can be reformulated as:
\begin{align}
- \log \mathcal{L}_{\mathrm{YJ}}(\lambda) = &\frac{n}{2}\log (2\pi)  - (\lambda-1)\sum_{i=1}^n \mathrm{sign}(x_i) \log (|x_i|+1)\nonumber \\
& + \ln \left(\frac{n_\alpha}{n_\gamma} \sigma_\alpha^2 + \frac{n_\beta}{n_\gamma} \sigma_\beta^2 + \frac{n_\alpha n_\beta}{n_\gamma^2} (\bar \alpha- \bar \beta)^2 \right).
\label{eq:splitlog}
\end{align}
The first term is constant and the second one is linear in $\lambda$
so we only have to prove the convexity of the last term to prove that
the full negative log-likelihood is convex.
Using \Cref{lemma:lnconvex}, we only need to show
that $\lambda \mapsto \ln \left(\frac{n_\alpha}{n_\gamma} \sigma_\alpha^2\right)$,
$\lambda \mapsto \ln \left(\frac{n_\beta}{n_\gamma} \sigma_\beta^2 \right)$
and $\lambda \mapsto \ln \left(\frac{n_\alpha n_\beta}{n_\gamma^2} (\bar \alpha- \bar \beta)^2 \right)$
are convex. We can get rid of the constant factor and show
that $\lambda \mapsto \ln \left(\sigma_\alpha^2\right)$,
$\lambda \mapsto \ln \left( \sigma_\beta^2 \right)$
and $\lambda \mapsto \ln \left((\bar \alpha- \bar \beta)^2 \right)$ are convex.

The key idea of the proof is to use the fact that, according
to \cite{kouider1995concavity}, for any set of positive
real numbers~$\{a_i\}$, $\lambda \mapsto \ln \sigma [\Phi(\lambda,  \{a_i\}]$ is convex,
where $\Phi(\lambda, \cdot)$ denotes the Box-Cox transformation. Besides we have (see \Cref{app:boxcox}):
\begin{align}
\alpha_i = \Psi(\lambda, x^+_i) &= \Phi(\lambda,  x^+_i  +1), \label{eq:trick1}\\ 
\beta_i = \Psi(\lambda, x^-_i) &= - \Phi(2-\lambda, 1 - x^-_i).\label{eq:trick2}
\end{align}
Therefore $\ln \sigma [\alpha_i] =  \ln \sigma [\Phi(\lambda,  (x^+_i +1)]$ which
is a convex function of $\lambda$.
Similarly, $\sigma[\{-\Phi(2-\lambda, 1 - x^-_i)\}]^2 = \sigma[\{\Phi(2-\lambda, 1 - x^-_i)\}]^2$.
The function $\lambda \mapsto \sigma[\{\Phi(2-\lambda, 1 - x^-_i)\}]^2$ is convex as the composition
of the linear function $\lambda \mapsto 2-\lambda$ with the convex
function $\lambda \mapsto \sigma[\{\Phi(\lambda, 1 - x^-_i)\}]^2$.

Let us finally prove the convexity of $\lambda \mapsto \ln \left[(\bar \alpha- \bar \beta)^2 \right]$.
We recall that $\bar \alpha > 0$ and $\bar \beta < 0$ and
that $\ln \left[(\bar \alpha- \bar \beta)^2 \right] =  2 \ln \left[\bar \alpha- \bar \beta \right]$.
Using \Cref{lemma:lnconvex}, we only need to prove
that~$\lambda \mapsto \ln \left(\bar \alpha\right)$
and~$\lambda \mapsto \ln \left(-\bar \beta\right)$ are convex.
As $\bar \alpha$ and $\bar \beta$ are defined as sums, still
using \Cref{lemma:lnconvex}, we only need to prove
that~$\lambda \mapsto \ln \left(\Psi(\lambda, x^+_i)\right)$
and~$\lambda \mapsto \ln \left(-\Psi(\lambda, x^-_i)\right)$ are convex for any $i$.
Using, Eqs.~\eqref{eq:trick1} and~\eqref{eq:trick2}, it is sufficient to prove
that for any real number $a \geq 1$, the
function~$\lambda \mapsto \ln [\Phi(\lambda,  (a)] = \ln [(a^\lambda-1)/\lambda]$ is convex,
which is proved in \Cref{app:ProofD}.

\end{proof}

\subsection{Proof that $\lambda \mapsto \ln [\Phi(\lambda,  (a)]$ is convex}
\label{app:ProofD}
Let $a \geq 1$ $u(\lambda) = (a^\lambda-1)/\lambda$ and $g(\lambda) = \ln u(\lambda)$.
For~$\lambda \neq 0$, the second derivative of $g$ is positive if and only if $D \defeq \lambda^4 (u u'' - (u')^2) \geq 0$.

We have
$$D(a, \lambda) = a^{2\lambda} - a^\lambda \lambda^2 \log(a)^2 - 2a^\lambda + 1.$$
Let us show that $D \geq 0$ when $\lambda \neq 0$.
$D(a=1, \lambda) = 0$, so we just need to show that $\partial_a D(a, \lambda) > 0$ when $a > 0$. As
$$\partial_a D(a, \lambda) = a^{(\lambda - 1)}\lambda(2a^\lambda - \lambda^2 \log(a)^2 - 2\lambda\log(a) - 2),$$
let us define $T(a,\lambda)$ as:
$$T(a, \lambda)= (2a^\lambda - \lambda^2 \log(a)^2 - 2\lambda\log(a) - 2).$$
We just need to show that $T(a, \lambda) > 0$ when $\lambda > 0$ and $T(a, \lambda) < 0$
when $\lambda < 0$. As $T(a, 0) = 0$, we just need to show that $\partial_\lambda T(a,\lambda) > 0$ when $\lambda \neq 0$.
$$\partial_\lambda T(a,\lambda) = 2(a^\lambda - \lambda\log(a) - 1)\log(a).$$
As $a > 1$, $\log(a) > 0$, so we just need to show that $(a^\lambda - \lambda\log(a) - 1) > 0$
which can be done by replacing $x$ by $\lambda\log(a)$ in the following inequality: $\exp(x) > x+1$ for $x >0$.

To conclude, when $\lambda \neq 0$ and $a \geq 1$, $D(\lambda, a) \geq 0$, and if $a > 1$ then $D(\lambda, a) > 0$. Therefore, the second derivative of $g$ is positive for any $\lambda \geq 0$. Using continuity, we can conclude that the second derivative of $g$ is positive for any $\lambda$ and that $\lambda \mapsto \ln [\Phi(\lambda,  (a)]$ is convex.

Note that if $a > 1$, then $D > 0$ and we can conclude that $\lambda \mapsto \ln [\Phi(\lambda,  (a)]$ is strictly convex.

\subsection{Edge cases not covered by the main proof of \Cref{th:convex}}
\label{app:edge}
In the main proof we assume that $n_\alpha \geq 2$, $n_\beta \geq 2$ and that $\sigma^2_\alpha > 0$, $\sigma^2_\beta > 0$. Said otherwise, we assume
that both $\{x_i^+\}$ and $\{x_i^-\}$ have at least two distinct elements. The proof is almost unchanged if this is not the case, as we can discard any term inside the logarithm of Eq.~\ref{eq:splitlog}. For example, let's assume that $n_\alpha = 1$. Therefore $\sigma^2_\alpha = 0$. We can then rewrite Eq.~\ref{eq:splitlog} as:
\begin{align}
- \log \mathcal{L}_{\mathrm{YJ}} = &\frac{n}{2}\log (2\pi)  - (\lambda-1)\sum_{i=1}^n \mathrm{sign}(x_i) \log (|x_i|+1)\nonumber \\
& + \ln \left(\frac{n_\beta}{n_\gamma} \sigma_\beta^2 + \frac{n_\alpha n_\beta}{n_\gamma^2} (\bar \alpha- \bar \beta)^2 \right).
\label{eq:splitlog2}
\end{align}
We only need to show
that $\lambda \mapsto \ln \left(\frac{n_\beta}{n_\gamma} \sigma_\beta^2 \right)$
and $\lambda \mapsto \ln \left(\frac{n_\alpha n_\beta}{n_\gamma^2} (\bar \alpha- \bar \beta)^2 \right)$
are convex as in the main proof.

Any other edge case can be treated similarly, and the proof holds as soon as $\{x_i\}$ has at least two distinct elements.

\subsection{Strict convexity of the Yeo-Johnson negative log-likelihood.}
\label{app:strictconvexity}
To prove the strict convexity of the YJ negative log-likelihood, let us notice that under the hypotheses of \Cref{lemma:lnconvex}, if at least one function $\lambda \mapsto \ln(f_i)$ is strictly convex, then $\lambda \mapsto \ln[\sum_i f_i(\lambda)]$ is strictly convex. Besides, according to \cite{kouider1995concavity}, for any set of positive
real $\{a_i\}$ with at least two distinct elements, $\lambda \mapsto \ln \sigma [\Phi(\lambda,  (a_i)]$ is strictly convex. Therefore, in the case where either $\{x^+_i\}$ or $\{x^-_i\}$ has two distinct elements, we can conclude that the YJ negative log-likelihood is strictly convex.

The only problematic case is when both $\sigma^2_\alpha = 0$ and $\sigma^2_\beta = 0$. In that case $\{x_i\}$ has only two distinct element: one positive or null and one strictly negative. In that case, $\lambda \mapsto \sigma[\{\Phi(\lambda, 1 - x^-_i)\}]^2$ is strictly convex as $\lambda \mapsto \ln [\Phi(\lambda,  (a)] = \ln [(a^\lambda-1)/\lambda]$ is strictly convex for $a> 1$.

\section{Secure Multi-Party Computation}
\label{app:SMC}
\subsection{Shamir Secret Sharing}
\label{app:SSS}
Secure Multiparty Computation (SMC) consists in evaluating functions without disclosing their inputs. One way to achieve this result is to use secret sharing. The main idea is that a value $h$ is split into different secret shares $h_k$, $k=1, \cdots , K$ where $K$ is the number of clients. Each client $k$ only knows the value of the secret share $h_k$, and one needs at least $p$ shares with $1 < p \leq K$ to recover the initial value $h$. The set of the secret shares $h_k$ of $h$ is denoted $\llbracket h \rrbracket$. Schematically, SMC consists in three main steps: (i) {\it secret sharing}, where each client splits its input into secret shares and sends them to the other clients (ii) {\it computation}, where the clients perform mathematical computations on the secret shares and obtain secret shares of the output and (iii) {\it reveal} steps, where the clients send each other the secret shares of the output in order to reconstruct and reveal the output.

In the Shamir Secret Sharing method \cite{shamir1979share}, the secret shares of $h$ correspond to the values of a given polynomial $P_h(x)$ of order $K$ at different points $x_k$ where $P_h(0) = h$. The values $x_k$ are arbitrarly chosen by the protocol with the constraint that all $x_k$ should be distinct. If all the clients disclose their secret share $h_k = P_h(x_k)$, then the secret $h$ can be recovered by polynomial interpolation. In this framework the addition can be done trivially. If $\llbracket h \rrbracket = \{h_k\}_{k=1,\cdots K}$ and $\llbracket g \rrbracket = \{g_k\}_{k=1,\cdots K}$ are the shares of $g$, then $\llbracket g+h \rrbracket = \{g_k+h_k\}_{k=1,\cdots K}$ are shares of $g+ h$. Said otherwise, $\llbracket g+h \rrbracket = \llbracket g \rrbracket + \llbracket h \rrbracket$. Therefore adding two shared secrets requires no communication between the clients. Similarly, multiplying a shared secret by a public constant $c$ is done without communication as $\llbracket c g\rrbracket = c \llbracket g\rrbracket$. However, multiplying two shared secrets, i.e.\ computing shares of $\llbracket g h \rrbracket$ is more involved and requires one round of communication. More precisely, each client has to send one scalar quantity to all the other clients during this process, as explained for example in \cite{reistad2012general}, section 3.

\subsection{Fixed-Point Representation}
\label{app:SMCFPR}
The secret shares in SMC belong to a finite set $\mathbb{Z}_p$ where $p$ is a prime number and all the operations are integer operations done modulo $p$. In practice we consider integers encoded using $l$ bits, then we choose the smallest prime number $p$ such that $2^l < p$ and we perform each operation modulo $p$. Therefore any value has to be encoded as an integer using a finite number of bits. To encode negative integers, we consider that encoded integers between $0$ and $2^{l-1}-1$ are positive and encoded integers between $2^{l-1}$ and $2^l -1$ are negative. We have to choose a value of $l$ large enough such that the highest absolute value considered is below $2^{l-1}$.
Real-value numbers are encoded using fixed-point precision, as described in \cite{catrina2010secure}, where the $f$ least significant bits of the encoding correspond to the decimal part, and the $l-f$ most significant bits correspond to the integer part. The addition of two fixed-point numbers in SMC can be done as described in \Cref{app:SSS}. However, multiplying two fixed-point representation numbers in SMC is more complex as the result must be divided by $2^f$, i.e.\ the $2^f$ least significant bits are discarded. As explained in detail in \cite{catrina2010secure}, multiplying two fixed-point numbers requires two rounds of communication (instead of one round of communication for the multiplication of two integers).

\subsection{Comparison in SMC}
\label{app:SMCcomparison}
In \secYJ\ , we need to compute in SMC the sign of an expression, which is equivalent to making a comparison with $0$. As we are using fixed-point representation encoding, computing the sign amounts to computing the most significant bit of the binary decomposition of a given shared secret. In order to do so, we use the method described in \cite{reistad2007secret}, which works for any SMC framework supporting addition and multiplication.  This method requires $10$ rounds of communication among the clients ($6$ of which can be done offline, i.e.\ they correspond to random values exchanged beforehand and can be done regardless of the value of the input). During these 10 rounds of communication,~$153l + 423 \log l+24$ multiplications are performed, $135l + 423 \log l+16$ of which can also be done offline. Notice that other SMC primitives could be used, such as the one described in \cite{escudero2020improved} which provides more efficient way to do SMC comparison. 

\subsection{MPyC}
\label{app:MPyC}
To implement \secYJ\ we used the python library MPyC \cite{schoenmakers2018mpyc}. MPyC is based built upon VIFF framework \cite{damgaard2009asynchronous} and is based on Shamir Secret Sharing \cite{shamir1979share}. We refer to \cite{lorunser2020performance} for a discussion of the performance of this library for various SMC tasks.

\subsection{Further details on \Cref{alg:SecureFedYJ}}
\label{app:fdAgo}
The pseudo-code provided in \Cref{alg:SecureFedYJ}, is a schematic overview of \secYJ and relies on the SMC routines described above. For example, the following line of the pseudo-code:

\begin{equation}
\llbracket S_\varphi \rrbracket = \sum_k \llbracket S_{k, \varphi}\rrbracket
\end{equation}
implies that: (i) each client $k$ computes $S_{k, \varphi}$, divide it into secrets and send these share secrets to all the other clients; (ii) Using the SMC routines described in \Cref{app:SSS,app:SMCFPR,app:MPyC}, the clients compute together the share secrets of $\llbracket S_\varphi \rrbracket$ where $S_\varphi = \sum_k S_{k, \varphi}$. After this step in \Cref{alg:SecureFedYJ}, the value of $S_\varphi$ is therefore shared using share secrets across all the clients. Notice that the server only plays an orchestration roles in this process.

   \begin{figure}
      \centering
        \includegraphics[width=0.48\textwidth]{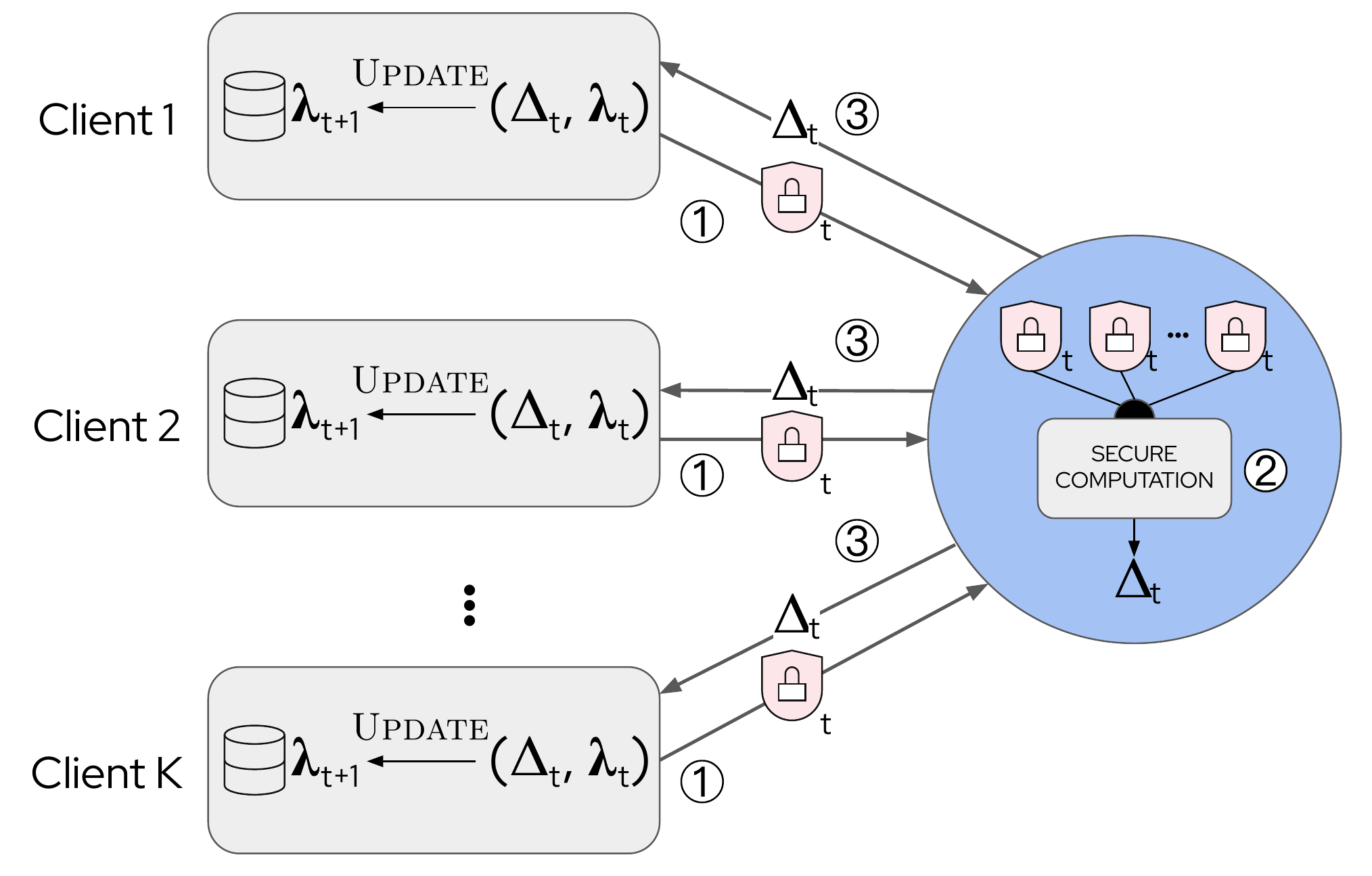}
\caption{Simplified view of one round of \secYJ . \numcircledmod{1}: All clients compute local, data-dependent quantities. \numcircledmod{2}: $\Delta_t$ is computed using SMC.  Data-dependent quantities computed by each client are not disclosed during the process. \numcircledmod{3}: $\Delta_t$ is disclosed to all the clients, a new value $\lambda_{t+1}$ is computed using an exponential update.
   }\label{schemaYJ}
\end{figure}

\subsection{Complexity of \secYJ\ }
\label{app:SMCrounds}
At each step of the exponential search, we share 6 secrets (the values of $S_g$), perform 10 fixed-point multiplications (including multiplying and dividing by $n$), and one comparison (i.e.\ computing the sign of $\partial_\lambda \mathcal{L}_{\mathrm{YJ}}$.

The 6 secrets can be shared in parallel in one round of communication. Some of the multiplications can also be done in parallel, and only 3 successive rounds of multiplications have to be performed, which require 6 rounds of communications. As stated in \Cref{app:SMCcomparison}, the comparison requires 10 rounds of communications. Revealing the secret $\Delta$ also requires one round of communication. Notice that the additions do not require any round of communication. This amounts to $18$ communications per exponential search step. Besides, computing $\llbracket S_\phi \rrbracket$ at the beginning of the algorithm and computing and revealing $\mu_*$ and $\sigma^2_*$ at the end of the algorithm requires 6 more rounds of communication. Overall, performing $40$ steps of exponential search with \secYJ\ costs $18\times 40 + 6 = 726$ rounds of communications.

For each elementary operation, such as sharing a secret, revealing a secret or making a multiplication, the order of magnitude of the size of the message sent by each client to the other clients is $\lceil \log_2(p) \rceil$ bits. Notice that $\log_2(p)$ is of the same order of magnitude of $l$ as $p$ is the smallest prime number above~$2^l$. More precisely, each client sends around $l$ bits to each of the other clients for these elementary operations. The overall size of the messages exchanged during the 726 rounds of communications mentioned above is mainly dominated by the $153l + 423 \log l+24$ multiplications done at each of the 40 comparisons. Taking $l = 100$, we find that each client sends overall around~$6.5\ 10^7$ bits (or $\sim 8$ Mega-bytes) to each of the other clients during \secYJ.

\section{Details of the numerical experiments}
\label{app:NR}
\subsection{Datasets used in this work}
\label{sec:datasets}
\paragraph{Datasets exposed by {\it scikit-learn} API used in \Cref{fig:brent_exp_comparison} and \Cref{fig:hetero_homo_comparison}}

For numerical experiments, we use four public datasets available in the UC Irvine Machine Learning repository \cite{Dua:2019} under a Creative Commons Attribution 4.0 International (CC BY 4.0) license and exposed by the {\it scikit-learn}
datasets API. These datasets are the {\it Iris dataset} \cite{fisher1936use} (150 samples, 4 features),
the {\it Wine Data Set} (178 samples, 13 features), the {\it Optical Recognition of
Handwritten Digits Data Set} (1797 samples, 64 features) and the
{\it Breast Cancer Wisconsin (Diagnostic) Data Set} (569 samples, 30 features). Only keeping features that have at least two distinct values, these datasets provide a total of 108 different features.

\paragraph{Extra UC Irvine Machine Learning repositories used to test Brent minimization method}

\paragraph{Genomic data used in \Cref{fig:tcga_exp}}
For genetic experiments, we rely on RNA-seq expression data from The Cancer Genome Atlas, expressed in Fragments per Kilobase Million (FPKM).
We focus on 3 cancers: colorectal cancer (COAD), lung cancer (LUAD + LUSC), and pancreatic adenocarcinoma (PAAD).
These datasets are available on \url{https://portal.gdc.cancer.gov/} under Open Access.

\subsection{Experiments on TCGA data}
\label{subsec:tcga_hyperparams}
Based on FPKM counts, we load all available data for each cancer of interest, removing genes with null expression for all samples.

\paragraph{Pipeline} Our pipeline consists of three steps:
\begin{enumerate}
     \item Normalization: either \textit{whitening}, \textit{log}, or Yeo-Johnson transformation;
     \item Dimensionality reduction: a PCA was applied on normalized data to reduce dimension
     (dimension $128$ for lung and colorectal cancer, $90$ for pancreatic cancer);
     \item Cox Proportional Hazards (CoxPH)~\cite{cox1972regression} model fitting.
\end{enumerate}

\paragraph{Normalization}
All normalization steps are performed on counts, regardless of the genes, as counts are related to the
same underlying phenomenon induced by next-generation RNA sequencing.
In other words, for the plain whitening, a single mean and variance is computed.
For \textit{log}, following application of $\log(1+\cdot)$ to all entries, a similar count-level whitening
is performed.
For the YJ transformation, we perform $10$ iterations of the proposed algorithm.

\paragraph{CoxPH model training}
CoxPH models are fitted with \textit{lifelines} (0.26.4).
We use an $\ell_2$ regularization of magnitude $10$ for each cancer, without any hyperparameter optimization.

\paragraph{Cross-validation}
Results are computed following $5$-fold stratified group cross-validation, repeated~$5$ times with different seeds.
Stratification is performed to ensure a balanced set of censored patients in each fold, while ensuring that samples
belonging to the same patients end up in the same group to avoid over-estimating the generalization of the model.

\subsection{Experiment on synthetic data}
\label{app:hyperparams}

To generate the results of \Cref{subfig:mock_1}, we sampled for each of the 10 centers 200 datapoints using Eq.~\eqref{eq:covariates}. We then apply an optional preprocessing steps before fitting a linear regression model using scikit-learn {\it LinearRegression} model on the pooled data. Another dataset of 200 points was then generated, and we computed the R2 on this unseen dataset. This experiment was repeated 1000 times using each time a different seed and the box plot in \Cref{subfig:mock_1} presents the min-max, the median the first and the third quartile. The different preprocessing steps shown are:

\begin{itemize}
\item None: no preprocessing step is applied
\item Whitening: for each of the three dimensions of $X_i$, we subtract the empirical mean and we divide by the empirical standard deviation computed across all ten centers to the train dataset and the test dataset
\item LocalYJ: we use one center randomly chosen to perform \expYJ\ with $t_\mathrm{max}=20$ to each of the dimensions of the dataset. The fitted triplets $\lambda_*, \mu_*, \sigma^2_*$ found for each column are then used to normalize the dataset of all 10 centers and the test dataset.
\item Federated YJ: We apply \secYJ\ with $t_\mathrm{max}=20$ on the 10 centers to each of the dimensions of the dataset. The fitted triplets $\lambda_*, \mu_*, \sigma^2_*$ found for each column are then used to normalize the dataset of all 10 centers and the test dataset.
\end{itemize}

\subsection{Testing Brent minimization on more dataset}
\label{app:Brent}
As explained in the paragraph {\it Numerical stability of \expYJ} of \Cref{sec:expYJ}, applying blindly the Brent minimization method of scikit-learn to minimize the Yeo-Johnson negative log-likelihood might result in numerical instabilities and might collapse all the values of the dataset into a single value. To check further whether this phenomenon is likely to appear, we apply the scikit-learn Yeo-Johnson transformation to various real-life tabular datasets of the UC Irvine Machine Learning repository \cite{Dua:2019} (which are under a Creative Commons Attribution 4.0 International, CC BY 4.0).  For each dataset, we only kept the features that have at least two distinct values. We found that for the 484 fetaures out of 24 datasets, this issue arises 5 times, as summarized by \Cref{table:brentextended}

\begin{table}[t!]
\centering
\sisetup{
    group-separator={,},
    group-minimum-digits=3,
    table-number-alignment = right,
    table-figures-integer = 6,
    detect-weight = true,
    detect-inline-weight = math
}
\resizebox{0.8\columnwidth}{!}{%
\begin{tabular}{cccc}
\toprule
Dataset name &\# of samples&\# of features (with at least two distinct values)& \# of instabilities of Brent minimization\\
\toprule
airfoil self noise & 1503 & 5 & 0 \\
\rowcolor{LightRed}
blood transfusion & 748 & 4 & 1 \\
boston & 506 & 13 & 0 \\
\rowcolor{LightRed}
breast cancer diagnostic & 569 & 30 & 2 \\
california & 20640 & 8 & 0 \\
climate model crashes & 540 & 18 & 0 \\
concrete compression & 1030 & 7 & 0 \\
concrete slump & 103 & 7 & 0 \\
connectionist bench sonar & 208 & 60 & 0 \\
connectionist bench vowel & 990 & 10 & 0 \\
\rowcolor{LightRed}
ecoli & 336 & 7 & 2 \\
glass & 214 & 9 & 0 \\
ionosphere & 351 & 34 & 0 \\
iris & 150 & 4 & 0 \\
libras & 360 & 90 & 0 \\
parkinsons & 195 & 23 & 0 \\
planning relax & 182 & 12 & 0 \\
qsar biodegradation & 1055 & 41 & 0 \\
seeds & 210 & 7 & 0 \\
wine & 178 & 13 & 0 \\
wine quality red & 1599 & 10 & 0 \\
wine quality white & 4898 & 11 & 0 \\
yacht hydrodynamics & 308 & 6 & 0 \\
yeast & 1484 & 8 & 0 \\
\bottomrule
\end{tabular}
}

\caption{Number of feature for which the {\it scikit-learn} implementation of Yeo-Johnson based on Brent minimization method fails for 24 different datasets available on the UC Irvine Machine Learning repository \cite{Dua:2019}. We only kept the features with at least two distinct values.}
\label{table:brentextended}
\end{table}

\section{Further details on \Cref{prop:leakage}}
\label{app:AppF}
\Cref{prop:leakage} states that all intermediate quantities of \secYJ\ can be recovered from its final result $\lambda_*$. We provide in \Cref{alg:RecoverInfo} a way to construct the function $\mathcal{F}$ introduced in \Cref{prop:leakage} that can perform this recovery.

\begin{algorithm}[ht]
   \caption{Function $\mathcal{F}$ recovering quantities revealed by \secYJ}
   \label{alg:RecoverInfo}
\begin{algorithmic}
   \REQUIRE  Hyperparameters $\lambda_{t=0}, \lambda^-_{t=0}, \lambda^+_{t=0}$ number of steps $t_\mathrm{max}$, $\lambda_*$
   \FOR{$t=1$ {\bfseries to} $t_\mathrm{max}$}
   \IF{$\lambda_{t-1} < \lambda_*$}
   \STATE $\Delta_t = 1$
   \ELSE
   \STATE $\Delta_t = -1$
   \ENDIF
   \STATE $\lambda_t, \lambda^-_t, \lambda^+_t \leftarrow \mathrm{\textsc{ExpUpdate}}(\lambda_{t-1}, \lambda^-_{t-1},\lambda^+_{t-1}, \Delta_{t})$
   \ENDFOR
   \ENSURE $(\lambda_t, \lambda^-_t, \lambda^+_t, \Delta_t)_{t=0,\dots, t_\mathrm{max}}$
\end{algorithmic}
\end{algorithm}

We apply \Cref{alg:RecoverInfo} on the 108 features used in \Cref{fig:hetero_homo_comparison}, with a fixed-point precision of $f=50$. We numerically check that the output of $\mathcal{F}$ from \Cref{alg:RecoverInfo} matches the intermediate quantities revealed by \Cref{alg:SecureFedYJ} up to machine precision.

\end{document}